\documentclass[journal]{IEEEtran}
\IEEEoverridecommandlockouts                              %

\newcommand\copyrighttext{%
  \scriptsize © 2023 IEEE. The paper will appear in the IEEE Transactions on Robotics. Personal use of this material is permitted.  Permission from IEEE must be obtained for all other uses, in any current or future media, including reprinting/republishing this material for advertising or promotional purposes, creating new collective works, for resale or redistribution to servers or lists, or reuse of any copyrighted component of this work in other works.
  }
\newcommand\copyrightnotice{%
\begin{tikzpicture}[remember picture,overlay]
\node[anchor=south,yshift=10pt] at (current page.south) {\fbox{\parbox{\dimexpr\textwidth-\fboxsep-\fboxrule\relax}{\copyrighttext}}};
\end{tikzpicture}%
}

\usepackage{amsmath,amsthm,amssymb, amsfonts}
\usepackage{mathtools}
\usepackage{graphicx, xcolor}
\usepackage{pdfpages}
\usepackage{bm}
\usepackage{url}
\usepackage{enumerate}
\usepackage{float}
\usepackage[sort, compress]{cite}
\usepackage{cases,balance}
\usepackage{cancel} %
\usepackage[pdftex, pdfstartview={FitV}, pdfpagelayout={TwoColumnLeft},bookmarksopen=true,plainpages = false, colorlinks=true, linkcolor=black, citecolor = black, urlcolor = black,filecolor=black, pagebackref=false,hypertexnames=false, plainpages=false, pdfpagelabels]{hyperref}
\usepackage[font=small]{caption}
\usepackage{subcaption}

\usepackage{array}
\usepackage{booktabs}
\usepackage{multirow} %

\usepackage{comment} %
\excludecomment{comments} %

\usepackage{algorithm}
\usepackage[noend]{algpseudocode} %
\definecolor{commentclr}{RGB}{110, 149, 204}
\newcommand{\algcomment}[1]{{\color{commentclr}// #1}}
\makeatletter
\newcommand\fs@spaceruled{\def\@fs@cfont{\bfseries}\let\@fs@capt\floatc@ruled
  \def\@fs@pre{\vspace{0.6\baselineskip}\hrule height.8pt depth0pt \kern2pt}%
  \def\@fs@post{\kern2pt\hrule\relax}%
  \def\@fs@mid{\kern2pt\hrule\kern2pt}%
  \let\@fs@iftopcapt\iftrue}
\makeatother

\usepackage{tikz}
\newcommand*\circled[1]{\tikz[baseline=(char.base)]{
            \node[shape=circle,draw,inner sep= 0.3pt] (char) {#1};}}

\newcommand\norm[1]{\left\Vert#1\right\Vert}

\newcommand{\rom}[1]{\uppercase\expandafter{\romannumeral #1\relax}}

\newcommand{\R}{\mathbb{R}}
\newcommand{\N}{\mathbb{N}}

\DeclareMathOperator*{\argmax}{arg\,max}

\newcommand{\eg}{e.g.}
\newcommand{\ie}{i.e.}
\newcommand{\inv}{^{-1}}

\newcommand{\tr}{^\top}
\newcommand*{\defeq}{:=}
\newcommand{\logdet}{\log\det}

\theoremstyle{plain}%

\newtheorem{prop}{Proposition}

\theoremstyle{remark}
\newtheorem{remark}{Remark}

\theoremstyle{definition}

\newtheorem{problem}{Problem}
\newtheorem{definition}{Definition}

\newcommand{\cctrl}{c^{\mathrm{ctrl}}}

\newcommand{\cstate}{c^{\mathrm{state}}}

\newcommand{\cd}{\texttt{CD}}
\newcommand{\dls}{\texttt{DLS}}
\newcommand{\cls}{\texttt{CLS}}
\newcommand{\lexchange}{\texttt{FindProposal}}
\newcommand{\nop}{\small{\texttt{NOP}}}

\newcommand{\lsoffset}{O} %
\newcommand{\cmax}{c^{\mathrm{max}}}
\newcommand{\sopt}{\mathcal{S}^*}
\newcommand{\sls}{\mathcal{S}^{\mathrm{ls}}}

\newcommand{\mi}{\mathbb{I}}

\newcommand{\dsct}[3]{{#1}_{#2, #3}}

\newcommand{\ctns}[2]{ {#1}_{#2} } %

\newcommand{\dctns}[2]{ \dot{#1}_{#2} } %

\newcommand{\ndctns}[3]{ {#1}_{#2}^{(#3)} } %

\newcommand{\ctnsax}[3]{ {#1}_{#2}^{\text{#3}} } %

\newcommand{\ulqr}{u^{\text{LQR}}}

\graphicspath{{Figs/}}

\graphicspath{{Algorithms/}}

\graphicspath{{Proofs/}}

\definecolor{lime}{HTML}{A6CE39}
\DeclareRobustCommand{\orcidicon}{%
	\begin{tikzpicture}
	\draw[lime, fill=lime] (0,0) 
	circle [radius=0.16] 
	node[white] {{\fontfamily{qag}\selectfont \tiny ID}};
	\draw[white, fill=white] (-0.0625,0.095) 
	circle [radius=0.007];
	\end{tikzpicture}
	\hspace{-2mm}
}

\foreach \x in {A, ..., Z}{%
	\expandafter\xdef\csname orcid\x\endcsname{\noexpand\href{https://orcid.org/\csname orcidauthor\x\endcsname}{\noexpand\orcidicon}}
}

\title{Energy-Aware, Collision-Free Information Gathering for Heterogeneous Robot Teams}

\author{Xiaoyi Cai\orcidA{}\hspace{-0.1cm},
        Brent Schlotfeldt\orcidB{}\hspace{-0.1cm},
        Kasra Khosoussi\orcidC{}\hspace{-0.1cm},\\
        Nikolay Atanasov\orcidD{}\hspace{-0.1cm},~\IEEEmembership{Member,~IEEE},
        George J. Pappas\orcidE{}\hspace{-0.1cm},~\IEEEmembership{Fellow,~IEEE},
        and Jonathan P. How\orcidF{}\hspace{-0.1cm},~\IEEEmembership{Fellow,~IEEE}%
\thanks{This research was supported in part by Boeing Research \& Technology and ARL DCIST CRA W911NF-17-2-0181.}
\thanks{$^1$X.\ Cai and J.\ P.\ How are with the Department of Aeronautics and Astronautics, Massachusetts Institute of Technology, Cambridge, MA 02139, USA.
	    {\tt\{xyc, jhow\}@mit.edu}.}
\thanks{$^2$B.\ Schlotfeldt is at Waymo.
	    {\tt brentsc@waymo.com}.}
\thanks{$^3$K.\ Khosoussi is with the Commonwealth Scientific and Industrial Research Organisation (CSIRO)
	    {\tt kasra.khosoussi@csiro.au}.}
\thanks{$^4$N.\ Atanasov is with the Electrical and Computer Engineering Department, University of California San Diego, La Jolla, CA 92093, USA.
	    {\tt natanasov@ucsd.edu}.}
\thanks{$^5$G.\ J.\ Pappas is with the GRASP Laboratory, University of Pennsylvania, Philadelphia, PA 19104, USA.
	    {\tt pappasg@seas.upenn.edu}.}
}

\begin{document}

\maketitle
\begin{abstract}
This paper considers the problem of safely coordinating a team of sensor-equipped robots to reduce uncertainty about a dynamical process, where the objective trades off information gain and energy cost. Optimizing this trade-off is desirable, but leads to a non-monotone objective function in the set of robot trajectories. Therefore, common multi-robot planners based on coordinate descent lose their performance guarantees. Furthermore, methods that handle non-monotonicity lose their performance guarantees when subject to inter-robot collision avoidance constraints. As it is desirable to retain both the \textit{performance guarantee} and \textit{safety guarantee}, this work proposes a hierarchical approach with a distributed planner that uses local search with a worst-case performance guarantees and a decentralized controller based on control barrier functions that ensures safety and encourages timely arrival at sensing locations. Via extensive simulations, hardware-in-the-loop tests and hardware experiments, we demonstrate that the proposed approach achieves a better trade-off between sensing and energy cost than coordinate-descent-based algorithms.

\end{abstract}


\copyrightnotice

\floatstyle{spaceruled}%
\restylefloat{algorithm}%

\section{Introduction}\label{sec:introduction}
\begin{figure} [t]
\centering
\includegraphics[width=\columnwidth]{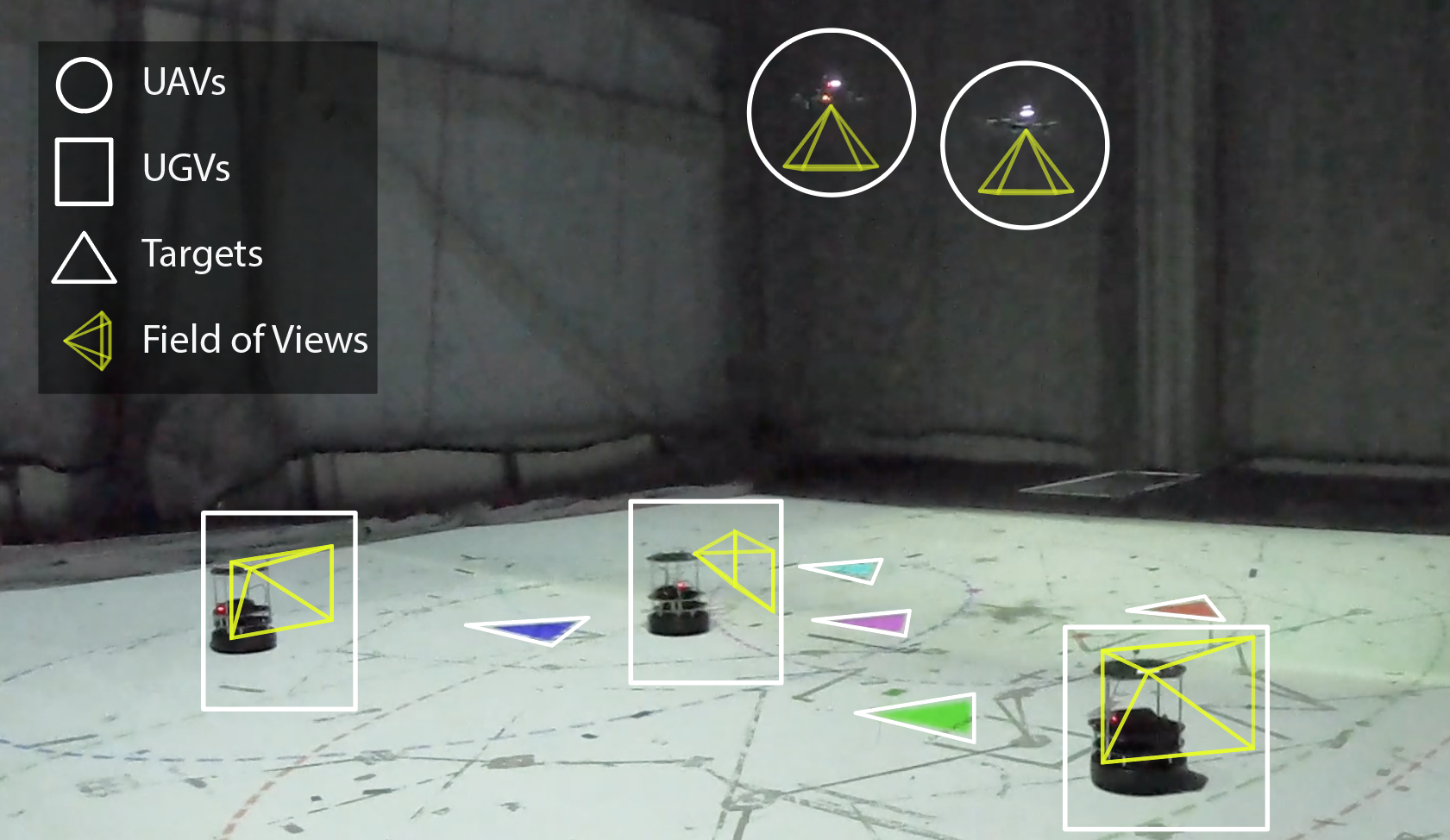}
\caption{A real-world dynamic target tracking scenario. Unmanned ground and aerial vehicles (UGVs and UAVs) use onboard cameras to collaboratively estimate the states of moving targets with different colors. The proposed approach is able to coordinate heterogeneous robots with desired trade-off between sensing performance and energy cost, while guaranteeing inter-robot collision avoidance.}
\label{fig:front_cover_hw_exp}
\end{figure}

Developments in sensing and mobility have enabled effective utilization of robot systems in diverse applications, such as autonomous mapping\cite{carlone2014active, atanasov2015decentralized, lopez2017multi, corah2019distributed}, agriculture~\cite{carbone2018swarm, edmonds2021efficient, dutta2021multi}, search and rescue\cite{kumar2004robot,beck2016online, tian2018search, queralta2020collaborative}, and environmental monitoring~\cite{shkurti2012multi, lan2016rapidly, notomista2020persistification, popovic2020informative, chen2021pareto}. 
These tasks require spatiotemporal information collection which can be achieved more efficiently and accurately by larger robot teams, rather than individual robots. Robot teams may take advantage of heterogeneous capabilities, require less storage and computation per robot, and may achieve better environment coverage in shorter time\cite{grocholsky2006cooperative,sukhatme2007design,tokekar2016sensor,rizk2019cooperative}.
Task-level performance is usually quantified by a measure of information gain, where typically the marginal improvements diminish given additional measurements (\textit{submodularity}), and adding new measurements does not worsen the objective (\textit{monotonicity}). Although planning optimally for multi-robot sensing trajectories is generally intractable, these two properties allow for \emph{near-optimal} approximation algorithms that scale to large robot teams, while providing worst-case guarantees~\cite{krause2014submodular}.

The work in this paper is motivated by problems in which robots seek to trade off between sensing performance and energy cost while maintaining safety constraints such as collision avoidance (an example shown in Fig.~\ref{fig:front_cover_hw_exp}).
Specifically, we formulate an \textit{energy-aware active information acquisition} problem, wherein the goal is to plan trajectories for a team of heterogeneous robots to maximize a weighted sum of information gain \textbf{and} the energy cost. Unlike imposing fixed budgets (e.g.,~\cite{singh2009efficient, singh2009nonmyopic , jorgensen2018team}), adding an energy cost breaks the \textit{monotonicity} of the objective, violating an assumption held by existing approximation algorithms. Thus, we present a distributed planning approach based on local search~\cite{lee2009non} that has a worst-case guarantee for the non-monotone objective function. 
For practical deployment, safety constraints such as inter-robot collision avoidance are required, but imposing such constraints directly during planning incurs large computational overhead and breaks the performance guarantee of existing approximation algorithms~\cite{corah2019distributed, kuhnle2019note}.
Therefore, we take a hierarchical approach that preserves the computational advantage and performance guarantee of the high-level planner by offloading safety constraints to the low-level controller. By leveraging control barrier functions (CBFs~\cite{ames2019control}), we propose a decentralized controller that ensures safety and encourages timely arrival at the sensing configurations via a novel use of weighted norm in the optimization objective.
Lastly, we discuss the conditions in which the high-level planner's performance guarantee will be preserved by the trajectories executed by the low-level controller. Our evaluation demonstrates that the overall approach achieves better trade-off between information gain and energy cost than the existing coordinate-descent-based methods in real-world multi-robot target tracking experiments.

\subsection{Related Work}
Multi-robot informative trajectory planning addresses the challenge of planning sensing trajectories for robots to reduce uncertainty about a dynamic process. To alleviate the computational complexity that scales exponentially in the number of robots (since the decision space is a Cartesian product of every robot's decision space), approximation methods have been developed to produce near-optimal solutions for a submodular and monotone objective (\eg, mutual information). A common technique is coordinate descent, where robots plan successively while incorporating the plans of previous robots. 
Coordinate descent was shown to extend the near-optimality of a single-robot planner to the multi-robot scenario~\cite{singh2009efficient}. 
This result was extended to dynamic targets~\cite{atanasov2015active} where each robot chooses from non-myopic sensing trajectories built from a search tree, achieving at least $50\%$ of the optimal performance regardless of the planning order. Similarly, the coordinate descent strategy was used to create an anytime method~\cite{schlotfeldt2018anytime} and a sampling-based method~\cite{kantaros2019asymptotically}.
Methods such as~\cite{jorgensen2018team, dames2017detecting} implemented the greedy method~\cite{fisher1978analysis} in a decentralized fashion, and a distributed version was proposed by~\cite{corah2019distributed} to alleviate the inefficiency in sequential planning.
Unlike prior work, this work considers the trade-off between the value and cost of information---an important formulation to consider for physical systems---but the objective becomes \textit{non-monotone} because additional sensing trajectories may not result in better performance due to high energy cost. As a result, existing techniques based on coordinate descent and greedy methods no longer have any worst-case performance guarantees. 

The problem can be seen as non-monotone submodular maximization subject to a partition matroid constraint, for which approximation algorithms already exist. Lee et al.~\cite{lee2009non} presents a method based on local search that can handle the intersection of multiple matroid constraints, where multiple solution candidates are generated and each solution set is built iteratively via additions or deletions. Extending~\cite{lee2009non}, a greedy-based approach~\cite{gupta2010constrained} was proposed to handle multiple independence systems, 
but has a worse approximation ratio given a single matroid. Other methods use multilinear relaxation~\cite{feldman2011unified, gharan2011submodular} for better approximation ratios, but require significant computation. In robotics, decentralized multi-robot task assignment was considered in~\cite{segui2015decentralised}, which adopted the continuous greedy method by~\cite{feldman2011unified}. The approach in~\cite{shin2019sample} combined sampling, greedy method, and lazy evaluation~\cite{minoux1978accelerated} to achieve fast computation.
The work presented here designs a distributed planner based on~\cite{lee2009non} for its simplicity and guarantees, while additionally incorporating well-known techniques, like greedy and lazy evaluation, to reduce the method's computation and communication.

Practical deployment of robots requires collision avoidance, for which existing approaches can be grouped into three categories. 
The first category involves \textit{spatial partitioning}, e.g., drones flying at different heights~\cite{schlotfeldt2018anytime} and ground robots driving in disjoint regions~\cite{di2021multi}, which leads to conservativeness as robots cannot easily work together to exploit different sensing modalities.
Methods in the second category adopt a \textit{hierarchical structure} that separates planning and control, where safety is only handled in low-level controller (e.g., \cite{dutta2019multi, woosley2020multi}). As a result, existing safe multi-robot controllers can be utilized (see~\cite{yan2013survey, verma2021multi} for comprehensive surveys) without affecting the problem structure in planning. 
The last category involves the \textit{joint optimization} of information gathering objectives and safety, by capturing collision avoidance as constraints or penalties, without necessarily assuming submodularity and monotonicity, where typical examples include~\cite{best2019dec, viseras2020distributed, gan2014online, zhang2017distributed}.
This work adopts a hierarchical approach that considers safety only in the low-level controller in order to preserve computational tractability and the performance guarantees of the planning algorithm. 
The discrete-time informative planning stage prescribes sequences of robot poses that need to be visited at specific time to achieve desired information-energy trade-off, and the continuous-time control stage needs to meet these terminal state and time conditions while also ensuring safety.
We propose a novel formulation of CBF-based controller with a weighted norm penalty that encourages timely arrival at planned sensing configurations by penalizing the deviation from the nominal mission rate captured by a time-varying control Lyapunov function (CLF). Although many works developed safe finite-time controllers using CBFs and CLFs (\eg,~\cite{garg2019control, garg2022fixed, garg2021robust, polyakov2022finite}), few have considered time-varying CLFs in weighted norms of the optimization objective.

\subsection{Contributions}
The energy-aware problem considered in the paper seeks to optimize the trade-off between information gain and energy cost while meeting safety requirements such as collision avoidance, for which existing approximation planning algorithms lose performance guarantees. By adopting a hierarchical approach that separates planning and control, our approach preserves the existing guarantee of a centralized planner adapted for distributed execution, while ensuring safety and encouraging timely arrival at sensing configurations via the controller. The contributions of this work are:
\begin{enumerate}
    \item a distributed planner based on local search that has an existing theoretical performance guarantee for a non-monotone submodular objective, with reduced computation and communication via lazy and greedy techniques,
    \item a decentralized controller that ensures safety via CBF and encourages timely arrival at desired sensing configurations with a novel use of weighted norm in the optimization objective that penalizes deviation from desired mission rate (Lyapunov function derivative),
    \item an extensive set of simulations, hardware-in-the-loop benchmarks and hardware experiments that demonstrate better ability to trade off sensing and energy costs than a state-of-the-art method while retaining practical feasibility in communication and computation.
\end{enumerate}

The preliminary version of this work appeared in~\cite{cai2021non} with the planning method and simulation results, while this work proposes a new decentralized controller that ensures safety and encourages arrival at sensing locations at designated time, provides extra simulation and hardware-in-the-loop benchmarks, and conducts hardware experiments that validate the proposed hierarchical approach's practical feasibility and performance.
Note that the proposed decentralized controller is not limited to information gathering tasks and may be applicable for other trajectory tracking problems.

\subsection{Outline}
The proposed planner and controller are introduced in Sec.~\ref{sec:multi_robot_planning} and~\ref{sec:control}.
Simulations and hardware-in-the-loop tests for the planner are in Sec.~\ref{sec:planning_results} and the controller is evaluated in Sec.~\ref{sec:control_results}, followed by hardware experiments with performance and feasibility analysis in Sec.~\ref{sec:hardware_experiments}.

\section{Energy-Aware Informative Trajectory Planning}\label{sec:multi_robot_planning}

This section considers how to generate trajectories for robots that optimize the trade-off between information gain and energy cost, using general nonlinear discrete dynamical and measurement models of robots and linear-Gaussian target motion models (Sec.~\ref{sec:planning_problem_formulation}). Because commonly used techniques such as coordinate descent~\cite{singh2009efficient, atanasov2015decentralized} no longer have a worst-case performance guarantee due to non-monotonicity in the problem, we propose to use local search~\cite{lee2009non} to provide near-optimal performance guarantees, which requires centralized computation undesirable in multi-robot application (Sec.~\ref{sec:cls}). Then, we propose a new distributed algorithm and reduce its communication and computation requirements based on greedy and lazy methods (Sec.~\ref{sec:dls}).

\subsection{Preliminaries}\label{sec:prelim}
We review some useful definitions. Let $g:2^\mathcal{M}\to \R$ be a set function defined on a ground set $\mathcal{M}$ consisting of finite elements. Let $g(a|\mathcal{S})\defeq g(\mathcal{S} \cup \{a\}) - g(\mathcal{S})$ be the marginal gain of $g$ at $\mathcal{S}$ with respect to $a$.

\begin{definition}[Submodularity]\label{def:submodularity}
Function $g$ is submodular if for any $\mathcal{S}_1\subseteq S_2\subseteq\mathcal{M}$ and $a\in\mathcal{M}\backslash \mathcal{S}_2$, $g(a|\mathcal{S}_1) \geq g(a|\mathcal{S}_2)$.
\end{definition}

\begin{definition}[Monotonicity]\label{def:monotonicity}
Function $g$ is monotone if for any $\mathcal{S}_1\subseteq \mathcal{S}_2 \subseteq \mathcal{M}$, $g(\mathcal{S}_1) \leq g(\mathcal{S}_2)$.
\end{definition}

\subsection{Planning Problem Formulation (Discrete Time)}\label{sec:planning_problem_formulation}

Consider robots indexed by $i \in \mathcal{R} \defeq \{1,\dots,n\}$, with states $x_{i, k}\in\mathcal{X}_i$ at step $k \in\{ 0,\ldots,K\}$ and dynamics:
\begin{equation}
\label{eq:robot_dynamics}
x_{i, k+1} = f_i(x_{i, k}, u_{i, k}),
\end{equation}
where $u_{i, k}\in\mathcal{U}_i$ is the control input and $\mathcal{U}_i$ is a finite set. We denote a control sequence as $\sigma_i = (u_{i,0}, \dots, u_{i,K-1}) \in \mathcal{U}_i^K$.

The robots' goal is to track targets with joint state $y\in\R^{d_y}$ that follows a linear-Gaussian motion model: 
\begin{equation}
    y_{k+1} = A_k y_k + w_k,\enspace\  w_k \sim \mathcal{N}(0,W_k),
    \label{eq:target_motion_model}
\end{equation}
where $A_k\in\R^{d_y \times d_y}$ and $w_k$ is a zero-mean Gaussian noise with covariance matrix $W_k \succeq 0$.
The robots have sensors that measure the target state subject to an observation model:
\begin{equation}
    z_{i,k} = H_{i,k }(x_{i,k})y_k+v_{i,k}(x_{i,k}),\enspace\  v_{i,k}\sim\mathcal{N}(0,V_{i,k}(x_{i,k})),
    \label{eq:sensor_observation_model}
\end{equation}
where $z_{i,k}\in\R^{d_{z_i}}$ is the measurement taken by robot $i$ in state $x_{i,k}$, $H_{i,k}(x_{i,k})\in\R^{d_{z_i} \times d_y}$, and $v_{i,k}(x_{i,k})$ is a state-dependent Gaussian noise, whose values are independent at any pair of time steps and across sensors.
The observation model is linear in target states but can be nonlinear in robot states. If it depends nonlinearly on target states, we can linearize it around an estimate of target states to get a linear model.

We assume every robot $i$ has access to $N_i$ control trajectories $\mathcal{M}_i = \{\sigma_i^\kappa\}_{\kappa=1}^{N_i}$ to choose from. Denote the set of all control trajectories as $\mathcal{M} = \cup_{i=1}^{n} \mathcal{M}_i$ and its size as ${N= |\mathcal{M}|}$.
Potential control trajectories can be generated by various single-robot information gathering algorithms such as~\cite{atanasov2014information, hollinger2014sampling, lan2013planning, levine2010information}.
The fact that every robot cannot execute more than one trajectory can be encoded as a partition matroid $(\mathcal{M}, \mathcal{I})$, where $\mathcal{M}$ is the ground set, and $\mathcal{I}= \{ \mathcal{S} \subseteq \mathcal{M} \mid |\mathcal{S} \cap \mathcal{M}_i| \leq 1\ \forall i\in\mathcal{R}\}$ consists of all admissible subsets of trajectories.
Given $\mathcal{S}\in\mathcal{I}$, we denote the joint state of robots that have been assigned trajectories as $x_{\mathcal{S},k}$ at time step $k$, and their indices as
$\mathcal{R}_\mathcal{S}\defeq \{i \mid |\mathcal{M}_i \cap \mathcal{S}| = 1\ \forall\ i\in\mathcal{R}\}$.
Also, denote the measurements up to step $k\leq K$ collected by robots $i\in\mathcal{R}_\mathcal{S}$ who follow the trajectories in $\mathcal{S}$ by $z_{\mathcal{S}, 1:k}$.

Due to the linear-Gaussian assumptions in \eqref{eq:target_motion_model} and  \eqref{eq:sensor_observation_model}, the optimal estimator for the target states is a Kalman filter. The target estimate covariance $\Sigma_{\mathcal{S},k}$ at time step $k$ resulting from robots $\mathcal{R}_\mathcal{S}$ following trajectories in $\mathcal{S}$ obeys:
\begin{equation}
    \Sigma_{\mathcal{S}, k+1} = \rho_{\mathcal{S},k+1}^{\mathrm e} ( \rho_{k}^{\mathrm p} (\Sigma_{\mathcal{S},k}), x_{\mathcal{S}, k+1}),
\end{equation}
where $\rho_{k}^{\mathrm p}(\cdot)$ and $\rho_{\mathcal{S},k}^{\mathrm e}(\cdot, \cdot)$ are the Kalman filter prediction and measurement updates, respectively:
\begin{equation*}
    \begin{aligned}
        \textbf{Predict:}&&\rho_{k}^{\mathrm p} (\Sigma) & \defeq A_k\Sigma A_k\tr +W_k, \\
        \textbf{Update:} && \hspace{-0.1in} \rho_{\mathcal{S},k}^{\mathrm e} (\Sigma, x_{\mathcal{S}, k}) & \defeq \left(\Sigma\inv + \sum_{i\in\mathcal{R}_\mathcal{S}} M_{i,k} (x_{i,k}) \right)\inv,\\
        && M_{i,k} (x_{i,k})& \defeq H_{i,k}(x_{i,k}) V_{i,k} (x_{i,k})\inv H_{i,k}(x_{i,k})\tr.
    \end{aligned}
\end{equation*}

When choosing sensing trajectories, we want to capture the trade-off between sensing performance and energy expenditure, which is formalized below.
\begin{problem}%
\label{prob:MA_fuel_aware}
Given initial states $x_{i,0}\in \mathcal{X}_i$ for every robot $i\in\mathcal{R}$, a prior distribution of target state $y_0$, and a finite planning horizon $K$, find a set of trajectories $\mathcal{S}\in\mathcal{M}$ to optimize:
\begin{equation} \label{eq:sensing_and_energy_objective}
\max_{\mathcal{S} \in \mathcal{I}}\  J(\mathcal{S}) \defeq  \mi(y_{1:K}; z_{\mathcal{S},1:K}) -  C(\mathcal{S}),
\end{equation}
where $\mi(y_{1:K}; z_{\mathcal{S},1:K})= \frac{1}{2}\sum_{k=1}^K \big[ \logdet\big( \rho_{k-1}^{\mathrm p} (\Sigma_{\mathcal{S},k-1})\big) - \logdet (\Sigma_{\mathcal{S},k}) \big] \geq 0 $ is the mutual information between target states and observations\footnote{Our problem differs from sensor placement problems that consider the mutual information between selected and not selected sensing locations.},
and $C:2^{\mathcal{M}}\to\R$ is defined as:
\begin{align}
C(\mathcal{S}) & \defeq \sum_{\sigma_i \in \mathcal{S}} m_i\,C_{i}(\sigma_i),\label{eq:fuel_C}
\end{align}
where $0\leq C_i(\cdot) \leq \cmax$ is a non-negative, bounded energy cost for robot $i$ to apply controls $\sigma_i$ weighted by $m_i \geq 0$.
\end{problem}

\begin{remark}
The robots are assumed to know others' states, motion models~\eqref{eq:robot_dynamics} and observation models~\eqref{eq:sensor_observation_model}, so that any robot can evaluate~\eqref{eq:sensing_and_energy_objective} given a set of trajectories.
\end{remark}

\begin{remark}
The optimization problem~\eqref{eq:sensing_and_energy_objective} is non-monotone, because adding extra trajectories may worsen the objective by incurring high energy cost $C(\mathcal{S})$. Thus, the constraint $\mathcal{S}\in\mathcal{I}$ may not be tight, \ie, some robots may not get assigned trajectories. This property is useful when a large repository of heterogeneous robots is available but only a subset is necessary for achieving good sensing performance.
\end{remark}

\begin{remark}
The choice of \eqref{eq:sensing_and_energy_objective} is motivated by the energy-aware target tracking application. However, the proposed algorithm in Sec.~\ref{sec:dls} is applicable to any scenario where  $J(\mathcal{S})$ is a submodular set function that is not necessarily monotone, but can be made non-negative with a proper offset.
\end{remark}

Solving Problem~\ref{prob:MA_fuel_aware} is challenging because adding energy cost $C(\mathcal{S})$ breaks the monotonicity of the objective, a property required for approximation methods, e.g., coordinate descent \cite{atanasov2015decentralized} and greedy algorithm \cite{fisher1978analysis}, to maintain performance guarantees. This is because these methods only add elements to the solution set, which always improves a monotone objective, but can worsen the objective in our setting, and may yield arbitrarily poor performance.

\begin{figure} [t]
\centering
\includegraphics[width=0.9\columnwidth]{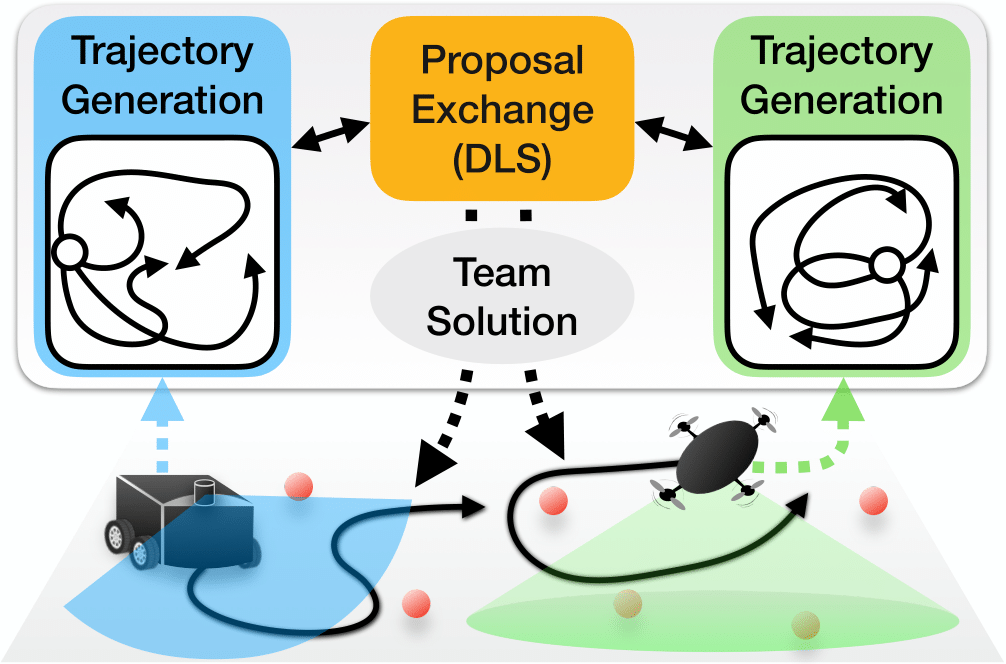}
\vspace*{-.05in}
\caption{Overview of the proposed distributed planning approach for non-monotone information gathering (see Sec.~\ref{sec:multi_robot_planning}). Robots generate individual candidate trajectories and jointly build a team plan via distributed local search (\dls{}), by repeatedly proposing changes to the collective trajectories.}
\label{fig:info_planning_vis}
\end{figure}

In the following subsections, we first present how local search~\cite{lee2009non} can be used to solve Problem~\ref{prob:MA_fuel_aware} with near-optimal performance guarantees but requires centralized computation. Subsequently, we propose a new distributed algorithm (see Fig.~\ref{fig:info_planning_vis}) that exploits the structure of a partition matroid to allow robots to collaboratively build a team plan and design techniques based on greedy and lazy methods to reduce its communication and computation requirements.

\subsection{Centralized Local Search (\cls{})}\label{sec:cls}
This section presents the original local search~\cite{lee2009non} in our setting with a single partition matroid constraint. We refer to it as centralized local search (\texttt{CLS}, Alg.~\ref{alg:cls}) because it requires access to trajectories $\mathcal{M}$ from all robots. 
We denote $g:2^{\mathcal{M}}\to \R$ as the non-negative, submodular oracle function used by local search, where the ground set $\mathcal{M}$ contains robot trajectories.
The algorithm proceeds in two rounds\footnote{Only two rounds are needed to maintain the worst-case performance guarantee given a partition matroid constraint~\cite{lee2009non}. Performance analysis for executing more planning rounds is useful but out of the scope of the paper.} to find two candidate solutions $\mathcal{S}_1, \mathcal{S}_2\in\mathcal{I}$. In each round $\kappa=1,2$, solution $\mathcal{S}_\kappa$ is initialized with a single-robot trajectory maximizing the objective (Line~\ref{cls:best_singleton}).
Repeatedly, $\mathcal{S}_\kappa$ is modified by executing one of the \textbf{Delete}, \textbf{Add} or \textbf{Swap} operations, if it improves the objective by at least $(1+\frac{\alpha}{N^4})$ of its original value (Lines~\ref{cls:start_local_ops}--\ref{cls:end_local_ops}), where $\alpha>0$ controls run-time and performance guarantees.
This procedure continues until $\mathcal{S}_\kappa$ is no longer updated, and the next round begins without considering $\mathcal{S}_\kappa$ in the ground set $\mathcal{M}$ (Line~\ref{cls:remove_sk}). 
Lastly, the better of $\mathcal{S}_1$ and $\mathcal{S}_2$ is returned.

\begin{algorithm}[t]
\caption{Centralized Local Search~\cite{lee2009non} (\cls)} 
\label{alg:cls}
\small %
\begin{algorithmic}[1]
\State \textbf{require} $\alpha>0$, ground set $\mathcal{M}$, admissible subsets $\mathcal{I}$, oracle $g$
\State $N {\footnotesize\leftarrow} | \mathcal{M} |$
\State $\mathcal{S}_1, \mathcal{S}_2 {\footnotesize\leftarrow} \emptyset$ 

\For{$\kappa=1, 2$}
    \State $\mathcal{S}_\kappa {\footnotesize\leftarrow} \{ \argmax_{a \in \mathcal{M} } g(\{ a \}) \}$ \hfill\algcomment{Initialize with best traj.} \label{cls:best_singleton} 
    \State \textbf{while} resultant $\mathcal{S}_\kappa'$ from \circled{1}, \circled{2} or \circled{3} satisfies $\mathcal{S}_\kappa'\in\mathcal{I}$ and $g(\mathcal{S}_\kappa') \geq (1+ \frac{\alpha}{N^4}) g(\mathcal{S}_\kappa)$ \textbf{do} $\mathcal{S}_\kappa {\footnotesize\leftarrow} \mathcal{S}_\kappa'$ \algcomment{Repeat local operations}\label{cls:start_local_ops}
    \State \hspace*{0.3em}  \circled{1} \textbf{Delete:} $\mathcal{S}_\kappa' {\footnotesize\leftarrow} \mathcal{S}_\kappa \backslash \{d\}$, where $d\in \mathcal{S}_\kappa$
    \State \hspace*{0.3em} \circled{2} \textbf{Add:} $\mathcal{S}_\kappa' {\footnotesize\leftarrow} \mathcal{S}_\kappa \cup \{a\}$, where $a\in \mathcal{M}\backslash \mathcal{S}_\kappa$
    \State \hspace*{0.4em} \circled{3} \textbf{Swap:} $\mathcal{S}_\kappa' {\footnotesize\leftarrow} \mathcal{S}_\kappa \backslash \{d\} \cup \{a\}$, where $d\in \mathcal{S}_\kappa,\,a\in \mathcal{M}\backslash \mathcal{S}_\kappa$\label{cls:end_local_ops}

    \State $\mathcal{M} {\footnotesize\leftarrow} \mathcal{M} \backslash \mathcal{S}_\kappa$ \algcomment{Update after while loop}\label{cls:remove_sk}
\EndFor

\State \textbf{return} $\argmax_{\mathcal{S} \in \{\mathcal{S}_1, \mathcal{S}_2\}} g( \mathcal{S})$ \label{cls:return}
\end{algorithmic}
\end{algorithm}

One important requirement of \texttt{CLS} is that the objective function $g$ is non-negative. With the objective from Problem~\ref{prob:MA_fuel_aware}, this may not be true, so we add an offset $\lsoffset$. The next proposition provides a worst-case performance guarantee for applying Alg.~\ref{alg:cls} to Problem~\ref{prob:MA_fuel_aware} after properly offsetting the objective to be  non-negative.

\begin{prop}\label{thm:LS_guarantee}
Consider that we solve Problem~\ref{prob:MA_fuel_aware} whose objective is made non-negative by adding a constant offset: 
\begin{equation}\label{eq:LS_optimization}
    \max_{\mathcal{S}\in\mathcal{I}} \ g(\mathcal{S}) \defeq J(\mathcal{S}) +\lsoffset,
\end{equation}
where $\lsoffset  \defeq 
\sum_{i=1}^{n}m_i \cmax$. 
Denote $\sopt$ and $\sls$ as the optimal solution and solution obtained by \cls{} (Alg.~\ref{alg:cls}) for~\eqref{eq:LS_optimization}, by using  $g(\cdot)$ as the oracle. We have the following worst-case performance guarantee for the objective:
\begin{equation}
    0 \leq g(\sopt) \leq 4(1+\alpha) g(\sls).
\end{equation}
\end{prop}
\begin{proof}
In~\eqref{eq:sensing_and_energy_objective}, mutual information is a submodular set function defined on measurements provided by selected trajectories~\cite{atanasov2015decentralized}.
Moreover, $C(\mathcal{S})$ is modular given its additive nature:
\begin{equation}
    C(\mathcal{S}) = \sum_{\sigma_i \in \mathcal{S}} m_i C_{i}( \sigma_i)\geq 0.
\end{equation}
Since mutual information is non-negative, \eqref{eq:LS_optimization} is a submodular non-monotone maximization problem with a partition matroid constraint. 
Setting $k=1$ and $\epsilon=\alpha$ in~\cite[Thm. 2]{lee2009non}, the proposition follows directly after rearranging terms.
\end{proof}

\begin{remark}
Having the constant $\lsoffset$ term in~\eqref{eq:LS_optimization} does not change the optimization in Problem~\ref{prob:MA_fuel_aware}, but ensures that the oracle used by \cls{} (Alg.~\ref{alg:cls}) is non-negative so that the ratio $(1+\frac{\alpha}{N^4})$ correctly reflects the sufficient improvement condition. 
\end{remark}

Despite the guarantee, \cls{} is not suitable for distributed robot teams, because it assumes access to all locally planned robot control trajectories which can be communication-expensive to gather. Additionally, running it naively can incur significant computation: in the worst case, \cls{} requires $\mathcal{O}(\frac{1}{\alpha}N^6\log(N))$ oracle calls\footnote{For 2 solution candidates, each requires $\mathcal{O}(\frac{1}{\alpha}N^4\log(N))$ local operations, and $N^2$ oracle calls to find each local operation in the worst case.}, where $N$ is the total number of trajectories~\cite{lee2009non}. 
To address this problem, we propose a new distributed algorithm that exploits the structure of a partition matroid to allow robots to collaboratively build a team plan by repeatedly proposing changes to the collective trajectories. Moreover, we develop techniques to reduce its computation and communication to improve scalability.

\subsection{Distributed Local Search (\texttt{DLS})} \label{sec:dls}
This section proposes a distributed implementation of local search
(see Algs.~\ref{alg:dls} and~\ref{alg:local_exchange} written for robot $i$).
Exploiting the structure of the partition matroid, \dls{} enables each robot to propose local operations based on its own trajectory set, while guaranteeing that the team solution never contains more than one trajectory for every robot.
All steps executed by \cls{} can be proposed in a distributed fashion, so \dls{} provides the same performance guarantee in Proposition~\ref{thm:LS_guarantee}.
By prioritizing search orders and starting with greedy solutions, we reduce computation and communication of \dls{}, respectively.

\subsubsection{Distributed Proposal}
Every proposal consists of two trajectories $(d,a)$, where $d$ is to be deleted from and $a$ is to be added to the solution set.
We also define a special symbol ``$\nop$'' that leads to no set operation, \ie, $\mathcal{S}_\kappa \cup \{\nop\} = \mathcal{S}_\kappa \backslash \{\nop\} = \mathcal{S}_\kappa$. 
Note that $(d,\nop)$, $(\nop, a)$ and $(d, a)$ are equivalent to the \textbf{Delete}, \textbf{Add} and \textbf{Swap} steps in \cls{}.

Every robot $i$ starts by sharing the size of its trajectory set $|\mathcal{M}_i|$ and its best trajectory $a^*_i \in \mathcal{M}_i$ in order to initialize $\mathcal{S}_\kappa$ and $N$ collaboratively (Alg.~\ref{alg:dls} Lines~\ref{dls:start_N_Sk}--\ref{dls:end_N_Sk}).
Repeatedly, every robot $i$ executes the subroutine \lexchange{} (Alg.~\ref{alg:local_exchange}) in parallel, in order to propose changes to $\mathcal{S}_\kappa$ (Alg.~\ref{alg:dls} Lines~\ref{dls:start_find_proposal}--\ref{dls:end_find_proposal}). Since any valid proposal shared by robots improves the objective, the first $(d, a)\neq (\nop, \nop)$ will be used by all robots to update $\mathcal{S}_\kappa$ in every round (Alg.~\ref{alg:dls} Lines~\ref{dls:start_receive_proposal}--\ref{dls:end_receive_proposal}).
We assume instantaneous communication, so robots always use a common proposal to update their copies of $\mathcal{S}_\kappa$. Otherwise, if delay leads to multiple valid proposals, a resolution scheme is required to ensure robots pick the same proposal.

In \lexchange{} (Alg.~\ref{alg:local_exchange}), an outer loop looks for potential deletion $d\in \mathcal{S}_\kappa$ (Alg.~\ref{alg:local_exchange} Lines~\ref{lexchange:start_delete}--\ref{lexchange:end_delete}). Otherwise, further adding $a\in\mathcal{M}_i$ is considered, as long as the partition matroid constraint is not violated (Alg.~\ref{alg:local_exchange} Lines~\ref{lexchange:start_check_partition_matroid}--\ref{lexchange:end_check_partition_matroid}). Next, we discuss how to efficiently search for trajectories to add.

\begin{algorithm}[t]
\caption{Distributed Local Search (\dls)}
\label{alg:dls}
\small %
\begin{algorithmic}[1]
\State \textbf{require} $\alpha>0$, trajectories $\mathcal{M}_i$, oracle $g$
\State Sort $\mathcal{M}_i$ in descending order based on $g(a|\emptyset)$ for all $a\in\mathcal{M}_i$\label{dls:pq_sort}
\State $\mathcal{S}_1, \mathcal{S}_2 \leftarrow \emptyset$
\For{$\kappa=1, 2$}
    \State Broadcast $| \mathcal{M}_i |$ and $a_i^*\in\mathcal{M}_i$ that maximizes $ g(\{a_i^* \})$\label{dls:start_N_Sk}
    \State $\mathcal{S}_\kappa \leftarrow \{a^*\}$, where $a^*\in \{ a_i^* \}_{i=1}^n$ maximizes $ g(\{a^* \})$
    \State $N \leftarrow \sum_{i=1}^n | \mathcal{M}_i |$\label{dls:end_N_Sk}

    \Repeat\label{dls:start_find_proposal}
        \State Run \lexchange($\mathcal{S}_\kappa, \mathcal{M}_i,\alpha,N, g$) in background\label{dls:run_find_proposal}
        \If{Receive $(d, a) \neq (\nop, \nop)$}\label{dls:start_receive_proposal}
            \State Terminate \lexchange{} if it has not finished
            \State $\mathcal{S}_\kappa \leftarrow \mathcal{S}_\kappa \backslash \{ d \} \cup \{a\}$\label{dls:end_receive_proposal}
        \EndIf
    \Until{Receive $(d, a) = (\nop, \nop)$ from all robots}\label{dls:end_find_proposal}
    
    \State $\mathcal{M}_i \leftarrow \mathcal{M}_i \backslash \mathcal{S}_\kappa$
\EndFor %
\State \textbf{return} $ \argmax_{ \mathcal{S} \in \{\mathcal{S}_1, \mathcal{S}_2\}} g( \mathcal{S})$
\end{algorithmic}
\end{algorithm}
\begin{algorithm}[t]
\caption{Find Proposal (\lexchange)}
\label{alg:local_exchange}
\small
\begin{algorithmic}[1]
\State 
        \textbf{require}
         $\mathcal{S}_\kappa$, %
         $\mathcal{M}_i$, %
         $\alpha>0$, $N$,
         $g$

\For{$d\in \mathcal{S}_\kappa$ or $d=\nop$}\ \algcomment{Delete $d$, or no deletion}\label{lexchange:start_delete}
    \State $\mathcal{S}_\kappa^- \leftarrow \mathcal{S}_\kappa \backslash \{d\}$ \label{lexchange:SK_minus}
    \State $\Delta \leftarrow (1+\frac{\alpha}{N^4}) g(\mathcal{S}_\kappa) - g(\mathcal{S}_\kappa^-)$\ \algcomment{$\Delta$: deficiency of $\mathcal{S}_\kappa^-$}\label{lexchange:deficiency_of_SK_minus}
    \If{$\Delta \leq 0$}
        \State \textbf{broadcast} $(d,\nop)$ \label{lexchange:end_delete}
    \EndIf
    \If{$\exists\ a\in \mathcal{S}_\kappa^-$ planned by robot $i$}\label{lexchange:start_check_partition_matroid}
        \State \textbf{continue} \algcomment{Cannot add due to partition matroid}\label{lexchange:end_check_partition_matroid}
    \EndIf
    \For{$a\in\mathcal{M}_i$ in sorted order} \algcomment{Add $a$}\label{lexchange:iter_sorted_order}
        \If{$g(a|\emptyset) < \Delta$}\label{lexchange:no_promising_trajs}
            \State \textbf{break} \algcomment{No $a\in\mathcal{M}_i$ will improve $\mathcal{S}_\kappa^-$ enough}\label{lexchange:break_if_no_promising_trajs}
        \EndIf
        \If{$g(a|\mathcal{S}_\kappa^-) \geq \Delta$}\label{lexchange:found_proposal} 
            \State \textbf{broadcast} $(d,a)$\label{lexchange:broadcast_proposal}
        \EndIf
    \EndFor
\EndFor
\State \textbf{broadcast} $(\nop, \nop)$

\end{algorithmic}
\end{algorithm}

\subsubsection{Lazy Search}
Instead of searching over trajectories in an arbitrary order, we can prioritize the ones that already perform well by themselves, based on $g( a | \emptyset)$ for all $a\in\mathcal{M}_i$ (Alg.~\ref{alg:dls} Line~\ref{dls:pq_sort}). 
Note that $\emptyset$ is the empty set, and the marginal gain $g( a | \emptyset)$ is equivalent to $g(\{a\})-g(\emptyset)$.
In this fashion, we are more likely to find trajectories that provide sufficient improvement earlier (Alg.~\ref{alg:local_exchange} Lines~\ref{lexchange:found_proposal}--\ref{lexchange:broadcast_proposal}). 
Note that $g( a | \emptyset)$ is typically a byproduct of the trajectory generation process, so it can be saved and reused.

This ordering also allows us to prune unpromising trajectories.
Given the team solution after deletion $\mathcal{S}_\kappa^- \defeq \mathcal{S}\backslash \{d\}$, the required marginal gain for later adding trajectory $a$ is
\begin{equation}\label{eq:suff_improv_a}
    g(a|\mathcal{S}_\kappa^-) \geq \Delta\defeq (1+\frac{\alpha}{N^4}) g(\mathcal{S}_\kappa) - g(\mathcal{S}_\kappa^-).
\end{equation}
We can prune any $a\in\mathcal{M}_i$ if $g( a | \emptyset)<\Delta$ based on the diminishing return property: because $\emptyset \subseteq \mathcal{S}_\kappa^- $, we know that $\Delta > g( a | \emptyset) \geq g( a | \mathcal{S}_\kappa^-)$, violating condition~\eqref{eq:suff_improv_a}.
Similarly, all subsequent trajectories $a'$ can be ignored, because their marginal gains $g(a'|\emptyset)\leq g(a|\emptyset)<\Delta$ due to ordering (Alg.~\ref{alg:local_exchange} Lines~\ref{lexchange:no_promising_trajs}--\ref{lexchange:break_if_no_promising_trajs}). 
Lastly, if an addition improves $\mathcal{S}_\kappa^-$ sufficiently, the proposal is broadcasted (Alg.~\ref{alg:local_exchange} Lines~\ref{lexchange:found_proposal}--\ref{lexchange:broadcast_proposal}). 

\subsubsection{Greedy Warm Start}
We observe empirically that a robot tends to swap its own trajectories consecutively for small growth in the objective, increasing communication unnecessarily. 
This can be mitigated by a simple technique: when finding local operations initially, we force robots to only propose additions to greedily maximize the objective, until doing so does not lead to enough improvement or violates the matroid constraint. Then robots resume Alg.~\ref{alg:local_exchange} and allow all local operations. By warm starting the team solution greedily, every robot aggregates numerous proposals with smaller increase in the objective into a greedy addition with larger increase, thus effectively reducing communication.

\section{Safe and Timely Control}\label{sec:control}
The output from the discrete-time planning stage (Sec.~\ref{sec:multi_robot_planning}) consists of pose sequences that have to be reached by every robot at designated time. This section considers how to find continuous-time control policies for reaching the sensing configurations in a timely fashion, subject to safety constraints such as but not limited to inter-agent collision avoidance.  We model continuous-time dynamics using integrator models that are applicable to a wide range of robots with differentially flat or feedback-linearized dynamics such as quadrotors and differential-drive robots (Sec.~\ref{sec:control_problem_formulation}) where the system states and inputs can be expressed via a small set of flat variables and their derivatives (in our case, the position and its higher derivatives)~\cite{fliess1995flatness}.
We leverage solutions to the linear quadratic regulator (LQR) with fixed boundary state and time to obtain the nominal control with finite-time convergence (Sec.~\ref{sec:LQR}), which can be modified to ensure safety via CBF in the form of a Quadratic Program (QP) that can be easily adapted to run in a decentralized fashion (Sec.~\ref{sec:cbf} and Sec.~\ref{sec:decentralized_cbf}). Our main novelty is a new CBF-QP formulation with weighted norm in the objective that encourages timely arrival under safety constraints by penalizing derivation from the desired mission rate represented by the time derivative of the Lyapunov function under the LQR policy (Sec.~\ref{sec:weighted_cbf_qp}).
Lastly, we discuss the conditions under which the high-level planner's worst-case performance guarantee will be preserved by the trajectories executed by the low-level controller (Sec.~\ref{sec:controller_impact}).
Note that the proposed controller is applicable for many problems that involve trajectory tracking and is not limited to the information gathering scenario considered in this paper.

\subsection{Control Problem Formulation (Continuous Time)}\label{sec:control_problem_formulation}

Consider robots with integrator models of order $r\geq 1$, where $x_i=[p_i\tr, \dctns{p}{i}\tr, \dots, (\ndctns{p}{i}{r-1})\tr  ]\tr \in\R^{3r}$ is the state of robot $i$ 
and $p_i=[\ctnsax{p}{i}{x}, \ctnsax{p}{i}{y}, \ctnsax{p}{i}{z}]\tr \in\R^3$ contains the x, y, z positions. The state for each robot $i$ evolves according to the following dynamical model:
\begin{equation}\label{eq:int_system}
\dot x_i = 
\overbrace{
\overbrace{
\begin{bmatrix}
0 & 1 & \cdots & 0 \\
0 & 0 & \ddots & \vdots \\
\vdots & \ddots & \ddots & 1 \\
0 & \cdots & \cdots & 0
\end{bmatrix}
}^{F\in \R^{r \times r}}
\otimes I_{3\times3} 
}^{A\in \R^{3r \times 3r}}
\cdot x_i
+ 
\overbrace{
\overbrace{
\begin{bmatrix}
0 \\ 0 \\ \vdots \\ 1
\end{bmatrix}
}^{G\in\R^{r}}
\otimes I_{3\times3} 
}^{B\in\R^{3r\times3}}\cdot u_i.
\end{equation}

Recall that the solution to the planning Problem~\ref{prob:MA_fuel_aware} consists of a sequence of actions $\sigma_i=(\dsct{u}{i}{0}, \dots, \dsct{u}{i}{K-1} )\in \mathcal{U}^K_i$ for each robot $i$. Based on the discrete motion model~\eqref{eq:robot_dynamics}, $\sigma_i$ also corresponds to a state sequence $( \dsct{x}{i}{0}, \dots, \dsct{x}{i}{K}) \in \mathcal{X}^{K+1}_i$ with  $x_{i,0}$ being the initial state of robot $i$. 
Denote the corresponding continuous time stamps as $t_k=k\tau$ for $k\in\{0,\dots,K\}$, where $\tau>0$ is the discrete time interval.
As we consider a team of robots gathering information in 3D space, we assume that there is an appropriate mapping $\phi_i:\mathcal{X}_i\rightarrow \R^{3r}$ to convert discrete states used by planner to reference goals for controller:
\begin{equation}
\ctnsax{x}{i}{ref}(t_k) \defeq \phi_i(\dsct{x}{i}{k}),\quad \forall i\in\mathcal{R},\, k \in \{ 0,\dots,K \}
\end{equation}
are the reference states that every robot $i$ has to reach at time $t_k$. For linear systems like~\eqref{eq:int_system}, the problem of reaching desired states at desired time is well studied, thus we adopt techniques from the optimal control literature.

\begin{problem}%
\label{prob:safe_and_timely_control}
Given the solution to the Problem~\ref{prob:MA_fuel_aware} that consists of discrete-time state sequences $(\dsct{x}{i}{0}, \dots, \dsct{x}{i}{K} )\in \mathcal{X}^{K+1}_i$  for robot $i\in\mathcal{R}$ at steps $k\in\{0,\dots,K\}$, find a control policy $u(x, t)$ such that the system~\eqref{eq:int_system} under the policy satisfies
\begin{enumerate}
\item \textbf{\textit{Timely Arrival}}: robots reach their reference sensing configurations $\ctnsax{x}{i}{ref}(t_k) = \phi_i(\dsct{x}{i}{k})$ at time $t_k=k\tau$, for all $i\in\mathcal{R}$ and steps $k\in\{0,\dots, K\}$.
\item \textbf{\textit{Safety Constraints}}:
robots remain within the safe set $\mathcal{C}_0$ defined as the superlevel set of some continuously differentiable function $h:\R^{3nr}\rightarrow \R$ where
\begin{equation}\label{eq:safe_set}
\mathcal{C}_0 = \{ x \in \R^{3nr}\ \vert\ h(x) \geq 0 \}.
\end{equation} 
\end{enumerate}
\end{problem}

\subsection{Fixed-Final-State Linear Quadratic Regulator}\label{sec:LQR}
This section reviews the results in~\cite{lewis2012optimal} on LQR with fixed boundary conditions. In order to reach the sensing configurations $\ctnsax{x}{i}{ref}(t_k)$ for each robot $i$ as designated by the informative trajectory planning algorithm, a sequence of LQR problems with fixed boundary conditions can be efficiently solved to produce the policy $u_i(x_i, t)$ given current state $x_i$ during time $t_{k}\leq t\leq t_{k+1}$ for $k\in\{0,\dots,K-1\}$. For energy efficiency, we want to find such policies that minimize control effort:%
\begin{equation}
\label{eq:lqr}
\begin{aligned}
\min_{u_i} \quad &  \frac{1}{2} \int_{t_{k}}^{t_{k+1}} u_i(x_i,t)\tr R u_i(x_i,t)\, dt, \\
\textrm{s.t.} \quad & \ctns{x}{i}(t_{k}) =\ctnsax{x}{i}{ref}(t_k),\\
                    & \ctns{x}{i}(t_{k+1}) =\ctnsax{x}{i}{ref}(t_{k+1}),\\
                    & \dot{x}_i = Ax_i + Bu_i,
\end{aligned}
\end{equation}
where $R\in\R^{3\times 3}$ is some positive definite weight matrix.
The optimal control for~\eqref{eq:lqr} is open-loop and has a closed-form expression for every interval $t_{k} \leq t \leq t_{k+1}$. However, the need for collision avoidance may prevent robots from closely following the state trajectory induced by the open-loop controller. As a simple remedy, the open-loop controller for robot $i$ can be recomputed at time $t$ based on the associated robot state $x_i$, leading to a time-varying policy of the form:
\begin{equation}
\label{eq:closedloop_lqr_control}
\ulqr_i(x_i, t) = R\inv B\tr e^{A\tr(t_{k+1}-t)}G_k(t)\inv d_{i, k}(x_i, t),
\end{equation}
where $d_{i, k}$ is the final state difference
\begin{equation}
\label{eq:final_state_difference}
    d_{i, k}(x_i, t) = \ctnsax{x}{i}{ref}(t_{k+1})-e^{A(t_{k+1}-t)} x_i,
\end{equation}
and $G_k$ is the weighted reachability Gramian defined as
\begin{equation}
\label{eq:reachability_gramian}
G_k(t) = \int_{t}^{t_{k+1}} e^{A(t_{k+1}-s)} B R\inv B\tr e^{A\tr(t_{k+1}-s)} ds .
\end{equation}
Note that the integral~\eqref{eq:reachability_gramian} is not necessary and $G_k(t)$ can be obtained in closed-form. 
First, one has to solve the following matrix Riccati equation 
\begin{equation}\label{eq:lyapunov_equation}
    \dot P(t) = AP(t) + P(t)A\tr + BR\inv B\tr, 
\end{equation}
for $t_k\leq t \leq t_{k+1}$, whose solution takes the form of
\begin{equation}\label{eq:lyapunov_equation_soln}
P(t) = e^{A(t-t_{k})} P(t_{k}) e^{A\tr(t-t_{k})} + G_k(t).
\end{equation}
As the goal is to compute $G_k(t)$, we notice that~\eqref{eq:lyapunov_equation_soln} becomes $G_k(t)=P(t)$ if we set the initial condition as $P(t_{k})=0$. Therefore, once $P(t)$ has been derived offline analytically, $G_k(t)$ can also be evaluated efficiently online.

For the energy-aware planning problem~\ref{prob:MA_fuel_aware} in Sec.~\ref{sec:multi_robot_planning}, it is useful for the planner to know the energy cost associated with the policy~\eqref{eq:closedloop_lqr_control}, which can be computed in closed-form after substituting~\eqref{eq:closedloop_lqr_control} and~\eqref{eq:reachability_gramian} into the objective of~\eqref{eq:lqr}:
\begin{equation}
\label{eq:optimal_lqr_objective}
J^*_{i,k} = \frac{1}{2} d_{i,k} (\ctnsax{x}{i}{ref}(t_{k}), t_{k})\tr G_k(t_{k})\inv d_{i}(\ctnsax{x}{i}{ref}(t_{k}), t_{k}),
\end{equation}
for robot $i$ between $t_{k}$ and $t_{k+1}$. The energy cost~\eqref{eq:optimal_lqr_objective} is exact when the policy~\eqref{eq:closedloop_lqr_control} is executed faithfully when the nominal policy does not violate safety constraints.

This section has introduced the policy~\eqref{eq:closedloop_lqr_control} for robots to follow the trajectories produced by the high-level planner (Sec.~\ref{sec:multi_robot_planning}), where the planner can also use~\eqref{eq:optimal_lqr_objective} as estimates for the actual energy cost. Next, we present the mathematical framework that encodes constraints such as inter-robot collision avoidance.

\subsection{Safety via Control Barrier Functions}\label{sec:cbf}
Denote $x=[x_1\tr, \dots, x_n\tr]\tr \in\R^{3nr}$ and $u=[u_1\tr, \dots, u_n\tr]\tr \in\R^{3n}$ as the aggregate state and control of all $n$ robots and $x^{\text{init}}$ as the initial condition. We write the aggregate dynamics as
\begin{equation}
\label{eq:agg_int_system}
\dot x = \underbrace{A\otimes I_{n\times n} \, x }_{f(x)} + \underbrace{B\otimes I_{n\times n}}_{g(x)} u,
\end{equation}
which is also a control-affine system. Given control-affine robot dynamics, CBFs are Lyapunov-like functions that can be used to guarantee collision-free maneuvers of the robot team. Specifically, safety is encoded as the forward invariance of a set: if the system starts in the set, it will not leave the set. 
The safe set of the robot team can be encoded by the superlevel set $\mathcal{C}_0$~\eqref{eq:safe_set} of some continuously differentiable function $h:\R^{3nr}\rightarrow \R$. %
For concreteness, we consider $h(x)$ with relative degree of $r\in\N$ where $r$ is the degree of the integrator system~\eqref{eq:agg_int_system}. In other words, the control $u$ only appears in the $r$-th time derivative of $h$ which can be expressed using Lie derivatives\footnote{The Lie derivative of $h$ with respect to $f$ is $L_fh(x)=\frac{\partial h(x)}{\partial x}f(x)$.}:
\begin{equation}
h^{(r)}(x,u) = L_f^{r}h(x)+L_gL_f^{r-1}h(x)u,
\end{equation}
where $L_gL_f^{r-1}h(x)\not=0$ and $$L_gL_fh(x)=\dots=L_gL_f^{r-2}h(x)=0.$$
As a result, the CBF formulations such as~\cite{ames2014control, wang2017safetytro} that require relative degree $1$ cannot be applied. Therefore, we adopt the Exponential Control Barrier Functions (ECBF~\cite{nguyen2016exponential, ames2019control}) for enforcing high relative-degree safety constraints and the forward invariance of $\mathcal{C}_0$. Next, we provide definitions for ECBF with the maximum relative degree $r$ permitted by our choice of system model~\eqref{eq:agg_int_system}.%

\begin{definition}[Exponential Control Barrier Function~\cite{nguyen2016exponential, ames2019control}]
\label{def:ECBF}
Given a set $\mathcal{C}_0\subset\R^{3nr}$ defined as the superlevel set of a $r$-times continuously differentiable function $h:\R^{3nr}\rightarrow \R$, then $h$ is an exponential control barrier function (ECBF) if there exists a row vector $K_\eta\in\R^{r}$ such that for the control system~\eqref{eq:agg_int_system},
\begin{equation}
\label{eq:ECBF_constraint_definition}
\sup_{u\in U} [L_{f}^{r}h(x) + L_gL_{f}^{r-1}h(x)u] \geq -K_\eta \eta(x),\forall x\in\mathcal{C}_0,
\end{equation}
resulting in $h(x(t))\geq C e^{(F-GK_\eta)t}\eta(x)\geq 0$ when $h(x^{\text{init}})\geq 0$, where ${\eta(x)=[h(x), \dot h(x),\dots, h^{(r-1)}(x)]\tr\in\R^{r}}$, {row vector $C=[1,0,\dots,0]\in\R^{r}$}, and $F$ and $G$ are defined in~\eqref{eq:int_system}.
\end{definition}
\begin{remark}
Note that the $K_\eta$ used in Def.~\ref{def:ECBF} is required to make the closed-loop matrix $F-GK_\eta$ that governs the evolution of $\eta$  have all strictly negative eigenvalues, which can be achieved via pole placement techniques from linear systems theory. Please refer to Sec.~\ref{sec:control_results} and~\ref{sec:hardware_experiments} for examples with $r=2$. 
Intuitively, the more negative the poles are, the more aggressive a robot's motion is allowed to be when the safety constraint is about to be violated.
\end{remark}

\subsection{Decentralized Safety Barriers}\label{sec:decentralized_cbf}

To account for inter-agent collision avoidance and the effects of propeller down-wash from the aerial vehicles, we follow~\cite{wang2017safedrones} and approximate the safety region of each vehicle $i$ as a cylinder-like super-ellipsoid. The safe set of each robot $i$ is implicitly defined by the set of positions $[p^{\text{x}}, p^{\text{y}}, p^{\text{z}}]\tr$ that satisfy:
\begin{equation}
\left[(\ctnsax{p}{i}{x}-p^{\text{x}})^2+(\ctnsax{p}{i}{y}-p^{\text{y}})^{2}\right]^2 + \left(\frac{\ctnsax{p}{i}{z}-p^{\text{z}}}{c}\right)^4 \leq D_s^4,
\end{equation}
where $D_s$ is the safety distance and $c$ is the z-axis scaling factor as visualized in Fig.~\ref{fig:circular_super_ellipsoid}. 
Note that the super-ellipsoid is appealing for robots flying tightly in 3D (see Fig.~\ref{fig:cbfqp_benchmark_example_scenario}) where it is desirable to maintain small horizontal distance but larger vertical clearance to reduce the effect of propeller down-wash.
For the safety of the entire robot team, every pair of distinct robots $i$ and $j$ has to respect each other's safety regions, leading to the following  position-based barrier function:
\begin{equation}
\label{eq:ecbf_collision}
h_{ij}(x_i, x_j) = 
\left[(\ctnsax{p}{i}{x}-\ctnsax{p}{j}{x})^2+(\ctnsax{p}{i}{y}-\ctnsax{p}{j}{y})^{2}\right]^2 + \left(\frac{\ctnsax{p}{i}{z}-\ctnsax{p}{j}{z}}{c}\right)^4 - D_s^4,
\end{equation}
which has a relative degree of $r$. The corresponding ECBF constraint is
\begin{equation}
\label{eq:ecbf_collision_constraint}
h^{(r)}_{ij} + K_\eta \cdot [h_{ij}, \dot h_{ij}, \dots, h^{(r-1)}_{ij} ] \geq 0,
\end{equation}
where $h^{(r)}_{ij}$ is affine in both $u_i$ and $u_j$. 
Enforcing~\eqref{eq:ecbf_collision_constraint} requires considering the controls from both robots $i$ and $j$ together, thus requiring central coordination. However, notice that the ECBF constraint defines the admissible control space, so it is possible to decentralize the constraint by partitioning the space such that every robot only uses control input from its own partition~\cite{wang2016safety}.
First, we introduce a useful structure in $h_{ij}^{(r)}$ that enables the distributed constraint enforcement.

\begin{figure} [t!]
\centering
\includegraphics[width=0.7\columnwidth, trim={0cm 0cm 0cm 0cm}, clip]{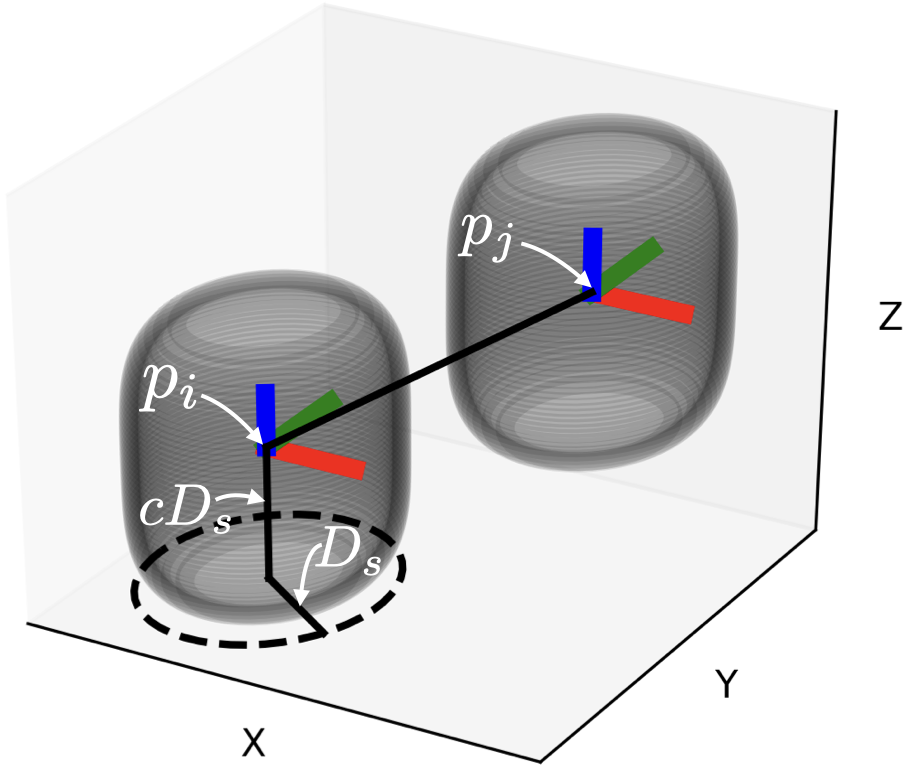}
\caption{Illustration of the safe set encoded by the circular super-ellipsoid barrier function~\eqref{eq:ecbf_collision} for robots $i$ and $j$ with position $p_{i}$ and $p_{j}$ respectively. Note that $D_s$ is the safety radius and $2cD_s$ is the height of the ellipsoid.}
\label{fig:circular_super_ellipsoid}
\end{figure}

\begin{prop}\label{prop:ecbf_affine_term}
Given system model~\eqref{eq:int_system} with order $r\geq 1$, the $r$-th derivative of the barrier function~\eqref{eq:ecbf_collision} can be written as
\begin{equation}
h_{ij}^{(r)} =  L_f^{r}h_{ij}(x_i,x_j)+ \underbrace{L_gL_f^{r-1}h(x_i,x_j)u}_{=A_{ij}(u_i - u_j)},
\end{equation}
where the row vector $A_{ij}\in\R^{3}$ has a fixed expression
\begin{equation}
\label{eq:ecbf_Aij}
A_{ij} = 4\left[(\Delta_x^2+\Delta_y^2)\Delta_x, (\Delta_x^2+\Delta_y^2)\Delta_y, \Delta_z^3/c\right],
\end{equation}
with $\Delta_x=\ctnsax{p}{i}{x}-\ctnsax{p}{j}{x}$, $\Delta_y=\ctnsax{p}{i}{y}-\ctnsax{p}{j}{y}$ and $\Delta_z=(\ctnsax{p}{i}{z}-\ctnsax{p}{j}{z})/c$ as the x, y position differences and scaled z position difference.
\end{prop}
\begin{proof}
See Appendix.~\ref{appendix:ecbf_affine_term_proof}.
\end{proof}

With Prop.~\ref{prop:ecbf_affine_term}, we now introduce the distributed ECBF constraints.
\begin{prop}
Given system model~\eqref{eq:int_system} with order $r\geq 1$ and an exponential control barrier function $h_{ij}$~\eqref{eq:ecbf_collision} that satisfies Def.~\ref{def:ECBF} with the associated row vector $K_\eta$, the inequality constraint~\eqref{eq:ecbf_collision_constraint} can be distributed to robots $i$ and $j$ as
\begin{align}
-A_{ij}u_i &\leq \frac{\alpha_i}{\alpha_i+\alpha_j} b_{ij}, \label{eq:dec_ecbf_constraint_ui}\\
A_{ij}u_j &\leq \frac{\alpha_j}{\alpha_i+\alpha_j} b_{ij}, \label{eq:dec_ecbf_constraint_uj}
\end{align}
where $\alpha_i, \alpha_j> 0$, $A_{ij}$ is defined in~\eqref{eq:ecbf_Aij} and
\begin{equation}
\label{eq:ecbf_bij}
b_{ij} = K_\eta\cdot \left[h_{ij}, \dot h_{ij}, \dots, h^{(r-1)}_{ij} \right] + L_{f}^{r}h_{ij}.
\end{equation}
Moreover, if the controllers $u_i$, $u_j$ for $i\not= j$ satisfy the decentralized ECBF constraints~\eqref{eq:dec_ecbf_constraint_ui} and~\eqref{eq:dec_ecbf_constraint_uj} and initial states $x_{i}^{\text{init}}, x_{j}^{\text{init}}$ satisfy $h_{ij}(x_{i}^{\text{init}}, x_{j}^{\text{init}})\geq 0$, then all robots are guaranteed to be safe.
\end{prop}
\begin{proof}
If decentralized constraints~\eqref{eq:dec_ecbf_constraint_ui} and~\eqref{eq:dec_ecbf_constraint_uj} are satisfied, then the centralized constraint~\eqref{eq:ecbf_collision_constraint} must be satisfied, which can be shown via simple addition of the left hand sides and right hand sides of~\eqref{eq:dec_ecbf_constraint_ui} and~\eqref{eq:dec_ecbf_constraint_uj}. As $h_{ij}$ is an ECBF satisfying Def.~\ref{def:ECBF} and $h_{ij}(x_{i}^{\text{init}}, x_{j}^{\text{init}})\geq 0$ for all $i\not=j$, all pair-wise (thus the team) safety will be guaranteed.
\end{proof}
With the decentralized ECBF constraints, every robot can compute its own control while satisfying $n-1$ pair-wise collision avoidance constraints; on the other hand, $\frac{n(n-1)}{2}$ pair-wise constraints are considered in the centralized case. The computational savings make decentralization attractive when computation is limited, but one must note that the decentralized ECBF constraints lead to smaller admissible control space for every robot, thereby causing infeasibility in extreme cases. However, this issue did not occur during our empirical evaluation, because we only considered relatively small robot teams. More discussion of this issue and potential resolutions can be found in the literature (e.g.,~\cite{wang2017safetytro}), but is not a central focus of this paper.

\subsection{Decentralized CBF-QP with Weighted Norm Penalty}\label{sec:weighted_cbf_qp}

We now present the proposed decentralized controller that encourages timely arrivals at desired sensing configurations while satisfying inter-robot collision avoidance constraints. For robot $i$ with state $x_i$ at time $t$, its control input $u_i$ should match the nominal time-varying policy $\ctnsax{u}{i}{LQR}(x_i, t)$~\eqref{eq:closedloop_lqr_control} as best as possible when safety is not at risk, which can be captured via the following penalty:
\begin{equation}\label{eq:deviation_nominal_control}
\norm{u_i - \ctnsax{u}{i}{LQR}(x_i, t)}^2.
\end{equation}
In addition, the robots should match the desired mission rate in order to arrive at the sensing configurations as timely as possible---a desirable property for information gathering since performance is time-sensitive. To capture the mission rate, we utilize the Lyapunov function associated with the LQR problem~\eqref{eq:lqr}:
\begin{equation}
\label{eq:lqr_V}
V_i(x_i, t) = \frac{1}{2} d_{i,k}(x_i,t)\tr G_{k}(t)\inv d_{i,k}(x_i,t),
\end{equation}
where $V_i:\R^{3r}\times\R\rightarrow\R$ is a time-varying function for robot $i$ given the final state difference $d_{i,k}(x_i,t)$ and the reachability gramian $G_{k}(t)$ for reaching the $k$-th sensing configuration. 
The following penalty captures the difference between the mission rates $\dot V_i$ (induced by $u_i$) and the optimal mission rate $\ctnsax{\dot V}{i}{LQR}$ (induced by the optimal controller $\ctnsax{u}{i}{LQR}$):
\begin{align}
&\norm{\dot V_i-\ctnsax{\dot V}{i}{LQR}}^2 \label{eq:Vdot_penalty}\\
=& \norm{\frac{\partial V_i}{\partial t}+\frac{\partial V_i}{\partial x_i}(Ax_i+Bu_i) -\frac{\partial V_i}{\partial t}-\frac{\partial V_i}{\partial x_i}(Ax_i+B\ctnsax{u}{i}{LQR})}^2 \nonumber\\
=& \norm{\frac{\partial V_i}{\partial x_i}B(u_i-\ctnsax{u}{i}{LQR})}^2, \label{eq:Vdot_penalty_intermediate}
\end{align}
where interestingly:
\begin{align}
\frac{\partial V_i}{\partial x_i}B &= -d_{i,k}(x_i,t)\tr G_{k}(t)\inv e^{A(t_{k+1}-t)}B \\
&= -(\ctnsax{u}{i}{LQR})\tr R,
\end{align}
where $R\succ 0$ is the weight matrix for penalizing control efforts~\eqref{eq:lqr}. Therefore, penalty~\eqref{eq:Vdot_penalty_intermediate} can be rewritten as:
\begin{align}
\norm{\frac{\partial V_i}{\partial x_i}B(u_i-\ctnsax{u}{i}{LQR})}^2
&=\norm{-(\ctnsax{u}{i}{LQR})\tr R(u_i-\ctnsax{u}{i}{LQR})}^2 \nonumber\\
&=\norm{u_i-\ctnsax{u}{i}{LQR}}_{\overline W}^2, \label{eq:Vdot_penalty_rank1}
\end{align}
where $\overline W=R\ctnsax{u}{i}{LQR}(\ctnsax{u}{i}{LQR})\tr R$ is a rank-$1$ matrix with spectral norm of $\norm{\overline W}_2 = (\ctnsax{u}{i}{LQR})\tr R^2\ctnsax{u}{i}{LQR}$.

In order to trade off the needs to match the nominal control and the desired mission rate, our proposed controller optimizes the weighted sum of~\eqref{eq:deviation_nominal_control} and~\eqref{eq:Vdot_penalty_rank1}, resulting in the following QP with the weighted norm $\norm{\cdot}_{W(\beta)}$:
\begin{align}
\min_{u_i} \quad & \norm{u_i - \ctnsax{u}{i}{LQR}(x_i, t)}_{W(\beta)}^2  \tag{\textbf{Weighted CBF-QP}}\label{eq:weighted_cbf_qp}\\
\textrm{s.t.} \quad & -A_{ij}u_i \leq \frac{\alpha_i}{\alpha_i+\alpha_j} b_{ij},\quad\forall j\not=i, \label{eq:weighted_cbf_qp_ecbf_constraint}%
\end{align}
where
\begin{equation}
W(\beta) = I_{3\times3} + \beta \overline W/\norm{\overline W}_2 \succ 0,\quad \beta \geq 0.
\end{equation}
The effect of $\beta$ is illustrated in Fig.~\ref{fig:weighted_cbfqp_illustration}. Importantly, when the constraints~\eqref{eq:weighted_cbf_qp_ecbf_constraint} are not active, the optimal solution matches the nominal control $\ctnsax{u}{i}{LQR}(x_i, t)$, because the objective is quadratic and $W$ is positive definite.
When safety constraints are active, increasing $\beta$ leads to solutions that better match the desired mission rate $\ctnsax{\dot V}{i}{LQR}$ rather than minimizing the deviation from the nominal control.

\begin{figure} [t!]
\centering
\includegraphics[width=0.8\columnwidth, trim={0cm 0cm 0cm 0cm}, clip]{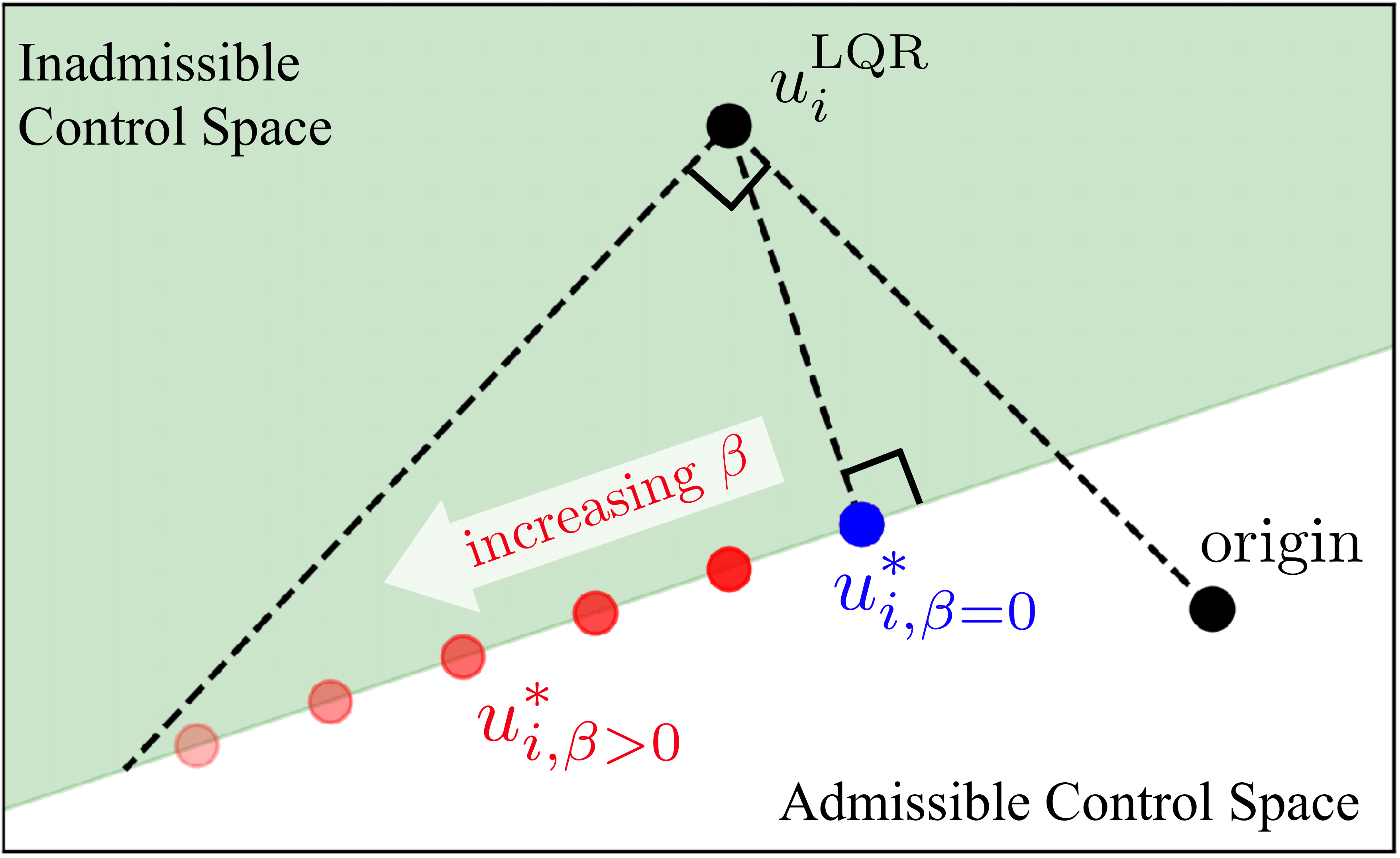}
\caption{Illustration of the effect of $\beta\geq 0$ on the solution of weighted CBF-QP, where $R$ is set as an identity matrix. When the nominal LQR constroller $\ctnsax{u}{i}{LQR}$ lies in the inadmissible control space, $u^*_{i,\beta=0}$ represents the solution of CBF-QP without weighted penalty which chooses the nearest admissible control in the $L_2$ sense. As $\beta$ increases, the optimization focuses more on matching the optimal mission rate $\ctnsax{\dot V}{i}{LQR}$; as a result, the vector component of $u^*_{i,\beta>0}$ that points in the direction of $\ctnsax{u}{i}{LQR}$ becomes greater and closer in magnitude with $\ctnsax{u}{i}{LQR}$ as $\beta$ approaches infinity if safety permits.%
}
\label{fig:weighted_cbfqp_illustration}
\end{figure}

\begin{remark}
Many additional control or state constraints can be added to the~\ref{eq:weighted_cbf_qp} formulation. For example, maximum infinity norm can be imposed over the control input, which is an affine constraint in the control. General box constraints can also be imposed via ECBF on each robot $i$ so that every entry of the state $x_i$ is limited to a range. For example, the position $p_i$ of a second order integrator should be limited within some bounded experiment area, and the velocity $\dot p_i$ should be bounded based on robot characteristics. 
\end{remark}

\subsection{Performance Guarantee of the Overall System}\label{sec:controller_impact}
The proposed low-level controller is used to execute the trajectories planned by the high-level planner proposed in Sec.~\ref{sec:dls}. The worst-case performance guarantee of the overall system is preserved only when the controller (a) achieves the desired sensing configurations in a timely fashion, and (b) consumes the same amount of energy as the energy cost used during planning. When the safety constraints are not active, we are able to not only satisfy (a) via the fixed-final-state LQR, but also achieve (b) by using the exact energy expenditure for each trajectory during planning thanks to the closed-form expression for energy cost~\eqref{eq:optimal_lqr_objective}. However, the performance guarantee is not preserved if the previous two conditions do not hold, for example, when the robots deviate from the nominal LQR trajectories to ensure safety. In this case, the proposed controller tries to reach the desired sensing configurations as timely as possible, with gracefully increasing energy expenditure as the mission becomes more challenging, as shown in the benchmark results in Sec.~\ref{sec:control_results} and Fig.~\ref{fig:cbfqp_benchmark_weighted_norm}.
\section{Analysis of Planning Algorithm}\label{sec:planning_results}
The proposed planning algorithm \dls{} is first analyzed in two target tracking simulations (Sec.~\ref{sec:exp_multi_dynamic_targets} and Sec.~\ref{sec:exp_hetero}) based on objective values, computation, communication, and ability to handle heterogeneous robots. Subsequently, a hardware-in-the-loop benchmark is conducted over a distributed network with delays (Sec.~\ref{sec:hil}). For all scenarios, $\dls{}$ is compared against coordinate descent (\cd{}~\cite{atanasov2015decentralized}), a state-of-the-art algorithm for multi-robot target tracking that, however, assumes monotonicity of the objective. Planning for robots sequentially, \cd{} allows every robot to incorporate the plans of previous robots. We also allow \cd{} to not assign anything to a robot if it worsens the objective. Reduced value iteration~\cite{atanasov2014information} is used to generate trajectories for both algorithms, {where its parameters $\epsilon,\delta\geq0$ are used to improve computational efficiency by pruning search nodes with similar information gain and close spatial proximity, respectively.}
Comparisons between \cls{} and \dls{} are omitted because the two algorithms empirically achieve the same average performance. We set $\alpha=1$ arbitrarily, because tuning it was not effective due to the large number of trajectories $N$.

Both \dls{} and \cd{} are implemented in C++ and evaluated in simulations on a laptop with an Intel Core i7 CPU.  For \dls{}, every robot owns separate threads, and executes Alg.~\ref{alg:local_exchange} over 4 extra threads to exploit its parallel structure. Similarly, \cd{} allows every robot to use 4 threads and additionally incorporates accelerated greedy~\cite{minoux1978accelerated} for extra speed-up. 
For the hardware-in-the-loop benchmark, we use $6$ laptops (robots) for running distributed computational nodes over a ROS network. 

\subsection{Characteristics of Simulated Robots}
Given initial state $x_{i,0} \in \mathcal{X}_i$ for robot $i\in\mathcal{R}_\mathcal{S}$ who follows the control sequence $ (u_{i,0},\dots,u_{i,K-1} )=\sigma_i\in \mathcal{S}$, the resultant states are $(x_{i,1}, \dots, x_{i,K})$ based on dynamics~\eqref{eq:robot_dynamics}. The energy cost $C(\mathcal{S})$ may also be state-dependent and it is defined as:
\begin{equation}\label{exp:energy_cost}
    C(\mathcal{S}) \defeq \sum_{i\in\mathcal{R}_\mathcal{S}} m_i \sum_{t=0}^{T-1} \left( \cctrl_i(u_{i,k}) + \cstate_i(x_{i,k}) \right),
\end{equation}
where the state-dependent cost $\cstate_i(\cdot)$ and control-dependent cost $\cctrl_i(\cdot)$ are defined based on robot types---in our case, robot $i$ is either an unmanned ground vehicle (UGV) or an unmanned aerial vehicle (UAV).  
Note that decomposition between state and control is not required for our framework to work. The setup for robots are summarized in Table~\ref{tab:robot_setup}. For simplicity, both the UGVs and UAVs follow differential-drive models for implementation convenience, with sampling period $\tau=0.5$ and motion primitives consisting of linear and angular velocities $\{u=(\nu, \omega) \mid \nu \in \{0,8\}\text{ m/s},\ \omega\in\{0,\pm\frac{\pi}{2}\}\text{ rad/s}\}$. We consider muddy and windy regions that incur state-dependent costs for UGVs and UAVs, respectively.
The robots have range and bearing sensors, whose measurement noise covariances grow linearly with target distance. Within limited ranges and field of views (FOVs), the maximum noise standard deviations are $0.1\text{ m}$ and $5^\circ$ for range and bearing measurements, respectively. Outside the ranges or field of views, measurement noise becomes infinite. %
See~\cite{schlotfeldt2018anytime} for more details.

\begin{table}[b]
    \caption{Robot setup in two simulations.}\label{tab:robot_setup}
    \centering
        \resizebox{\columnwidth}{!}{%
            \bgroup
            \def\arraystretch{1.3}%
            \begin{tabular}{@{\extracolsep{2pt}} m{0.3cm} ccc cc c c @{}} %
            \toprule
            \multirow{2}{*}{} & \multicolumn{3}{c}{$\cctrl(u)$, $u$ given as} & \multicolumn{2}{c}{$\cstate(x)$, $x$ in} & FOV ($^\circ$) & \multicolumn{1}{c}{Range (m)}\tabularnewline
            \cline{2-4} \cline{5-6} \cline{7-7} \cline{8-8} 
             &$0,0$ & $0,\frac{\pm\pi}{2}$ & $8,\frac{\pm\pi}{2}$ & Mud & Wind & Exp.1\&2 & Exp.1\&2 \tabularnewline 
            \midrule
            UGV & 0 & 1 & 2 & 3 & / & 160 & 6\ \&\ 15\tabularnewline
            UAV & 2 & 2 & 4 & / & 3 & 360 & /\ \&\ 20\tabularnewline
            \bottomrule
            \end{tabular}
            \egroup%
        }
\end{table}

\subsection{Simulation 1: Multi-Robot Dynamic Target Tracking}\label{sec:exp_multi_dynamic_targets}
Here we show the computation and communication savings for \dls{}, and compare the performance of \dls{} and \cd{} (see Figs.~\ref{fig:dls_analysis} and \ref{fig:dls_obj_control_sensing_time}).
The scenario involves $2$--$10$ UGVs trying to estimate the positions and velocities of the same number of dynamic targets. The targets follow discretized double integrator models corrupted by Gaussian noise, with a top speed of $2$~m/s. Robots and targets are spawned in a square arena whose sides grow from $40\text{ m}$ to $60\text{ m}$, and $50$ random trials are run for each number of robots.

\begin{figure} [t!]
\centering
\includegraphics[width=0.9\columnwidth, trim=10 0 45 13, clip]{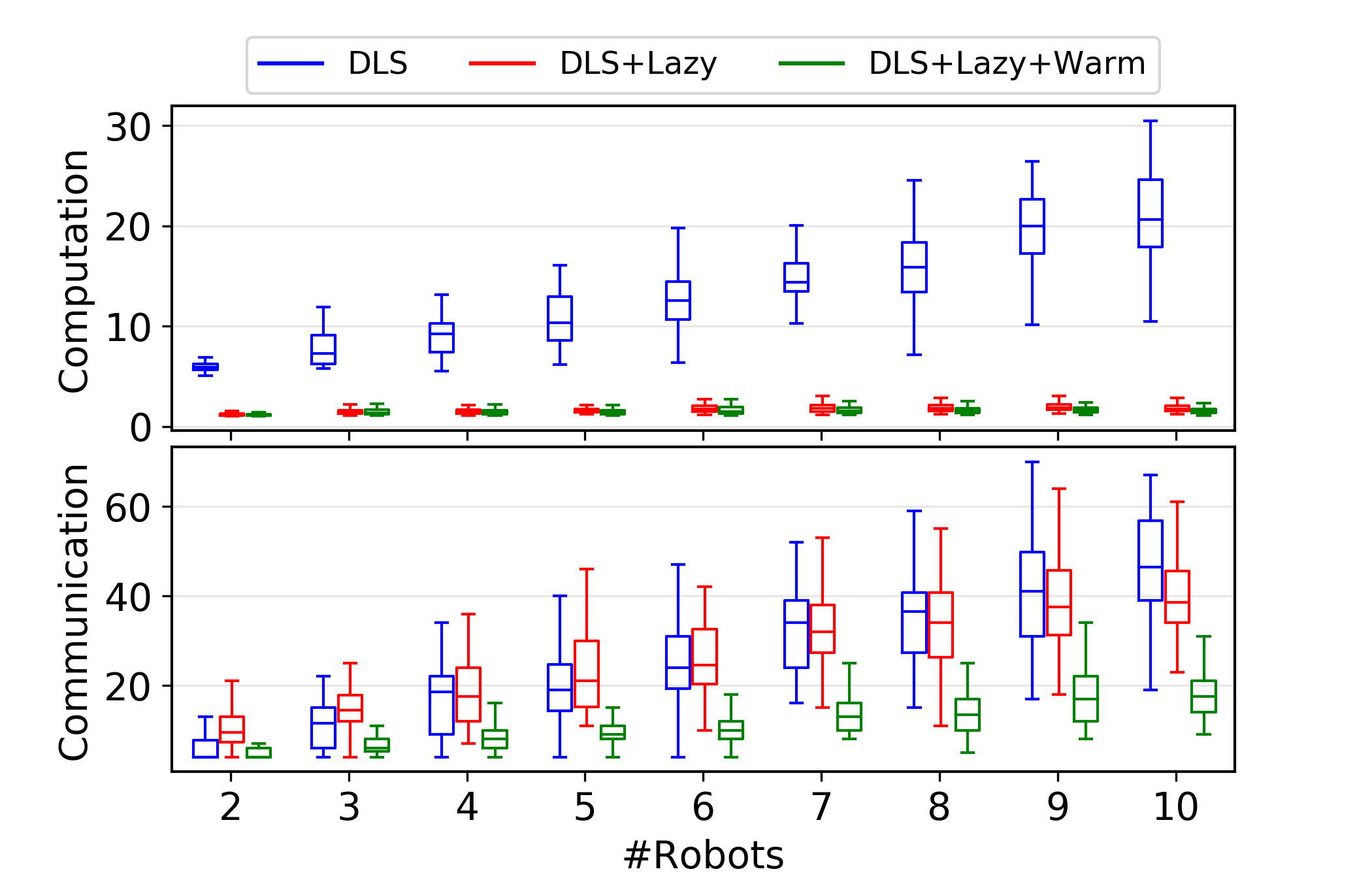}
\caption{Computation and communication savings afforded by lazy search (\textit{Lazy}) and greedy warm start (\textit{Warm}) for \dls{}.
Computation is measured by total oracle calls divided by the number of trajectories $N$, where $N$ reaches around $12500$ for $10$ robots. Communication is measured by the number of proposal exchanges. 
Combining lazy search and greedy warm start (green) leads to 80--92\% computation reduction, and up to 60\% communication reduction compared to the naive implementation (blue) on average.
}\label{fig:dls_analysis}
\end{figure}
\begin{figure} [t!]
\centering
\includegraphics[width=0.99\columnwidth, trim=20 15 20 0, clip]{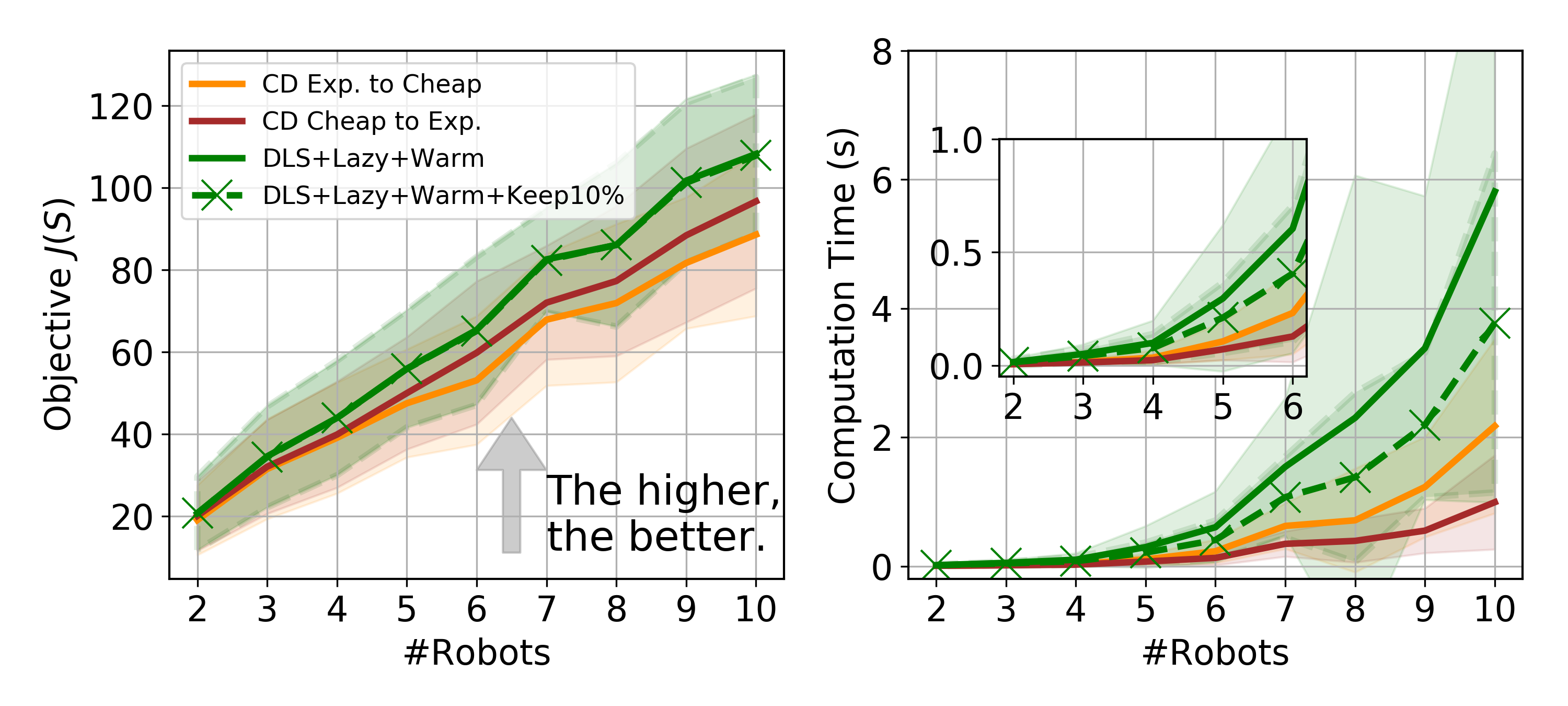}
\caption{
Objective values and computation time (s) for variants of \dls{} and \cd{}, where the lines and shaded areas show the mean and standard deviation, respectively. The time excludes the trajectory generation time ($<2$\text{ s}), which is the same for every algorithm.
\dls{} (solid green) consistently outperforms \cd{} in optimizing the objective, where it is better for \cd{} to plan from cheaper to more expensive robots (brown), rather than the reverse order (orange). The performance gap between \dls{} and \cd{} widens as more costly robots increase non-monotonicity of the problem.
However, \dls{} requires longer run-time, which in practice can be alleviated by using a portion of all trajectories. This invalidates its worst-case guarantee, but \dls{} solution based on the best $10\%$ of each robot's trajectories (green crosses) still outperforms \cd{}.}\label{fig:dls_obj_control_sensing_time}
\end{figure}

Non-monotonicity in the problem is accentuated by an increasing penalty for control effort of additional robots,
by setting  $m_i=i$ for each robot $i$ as defined in~\eqref{exp:energy_cost} (\ie, the $10\text{-th}$ added robot is $10$ times more expensive to move than the first). Note that state-dependent cost is set to $0$ for this experiment.
Trajectory generation has parameters $\epsilon=1$ and $\delta=2$ for horizon $T=10$. 
As the planning order is arbitrary for \cd{}, we investigate two planning orders: first from cheaper to more expensive robots, and then the reverse.
Intuitively and shown in Fig.~\ref{fig:dls_obj_control_sensing_time}, the former should perform better, because the same amount of information can be gathered while spending less energy. While other orderings are possible (\eg,\cite{jorgensen2018team,dames2017detecting}), we only use two to show \cd{}'s susceptibility to poor planning order. 
For a fair comparison between \dls{} and \cd{}, we use a fixed set of trajectories generated offline, but ideally trajectories should be replanned online for adaptive dynamic target tracking.

Proposed methods for improving naive distributed execution of local search, namely lazy search (\textit{Lazy}) and greedy warm start (\textit{Warm}), are shown to reduce computation by 80--92\% and communication by up to 60\% on average, as shown in Fig.~\ref{fig:dls_analysis}. 
As expected, when there are few robots with similar control penalties, the trade-off objective is easy to optimize, and \dls{} and \cd{} perform similarly as seen in Fig.~\ref{fig:dls_obj_control_sensing_time}. However, as more costly robots are added, their contributions in information gain are offset by high control penalty, making the problem harder to optimize. Therefore, the performance gap between \dls{} and \cd{} widens, because \cd{} requires monotonicity to maintain its performance guarantee, but \dls{} does not.
From Fig.~\ref{fig:dls_obj_control_sensing_time}, we can see that planning order is critical for \cd{} to perform well, yet a good order is often unknown prior to a mission.
Compared to \cd{} which requires only $n-1$ communication rounds for $n$ robots, \dls{} requires more for its performance. 
For practical concerns to save more time, \dls{} with down-sampled trajectories (\eg, keeping the best $10\%$ of each robot's trajectories) still produces better solution than \cd{}, but the guarantee of \dls{} no longer holds.

\subsection{Simulation 2: Heterogeneous Sensing and Control} \label{sec:exp_hetero}

\begin{figure} [tp!]
\centering
\includegraphics[width=0.7\columnwidth, trim=0 0 20 25, clip]{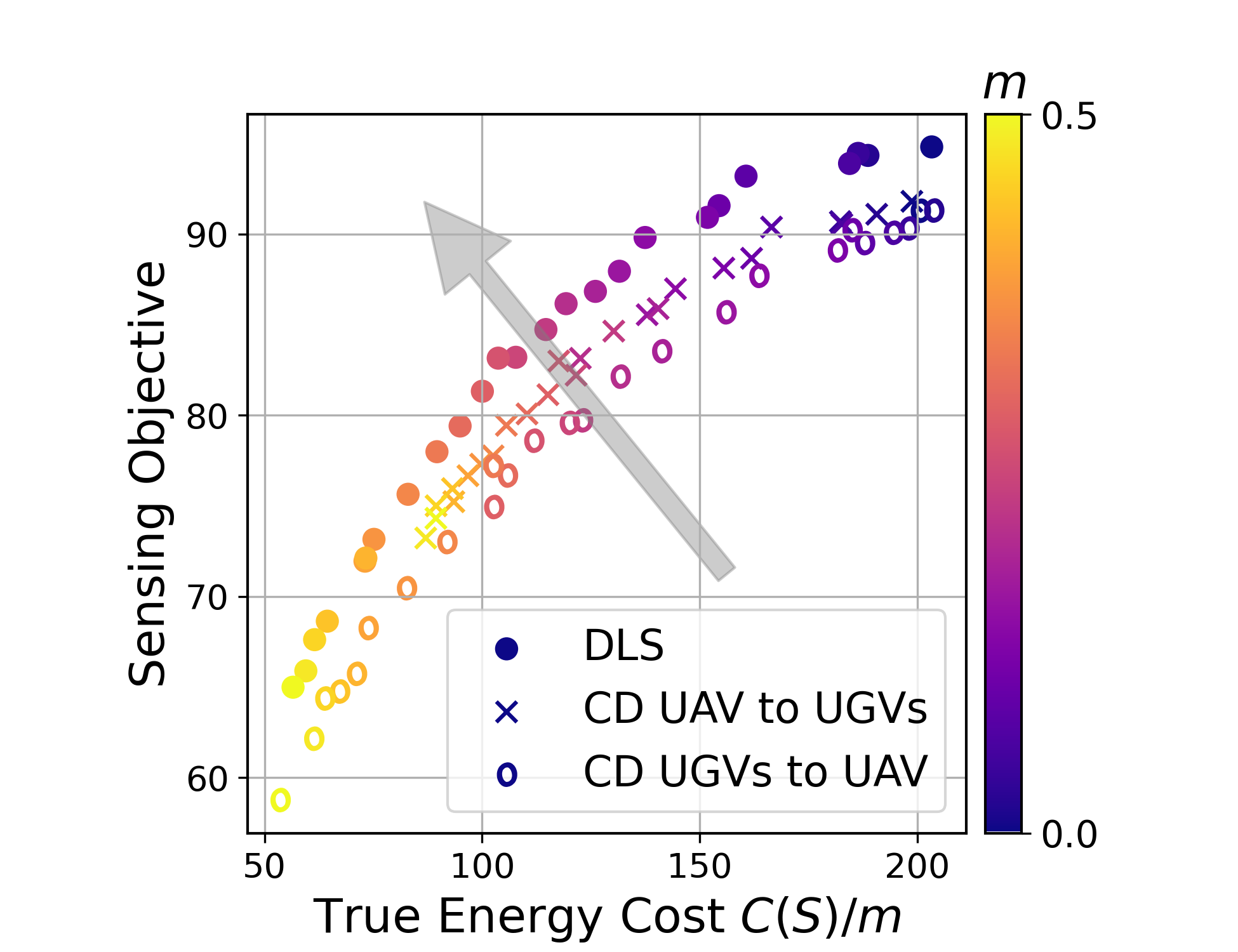}
\caption{Trade-off between sensing performance (mutual information~\eqref{eq:sensing_and_energy_objective}) and the true energy expenditure $C(\mathcal{S})/m$ in heterogeneous robot experiments produced by \dls{} and \cd{}, where it is better to be in the upper left pointed by the gray arrow. Each point is an average obtained over $50$ trials for a fixed $m$, where we set $m_i=m$ for each robot $i$ to penalize the team energy expenditure per~\eqref{exp:energy_cost}.}
\label{fig:hetero_tradeoff}
\end{figure}

\begin{figure} [tp!]
\centering
\includegraphics[width=0.95\linewidth, trim=85 55 65 30, clip]{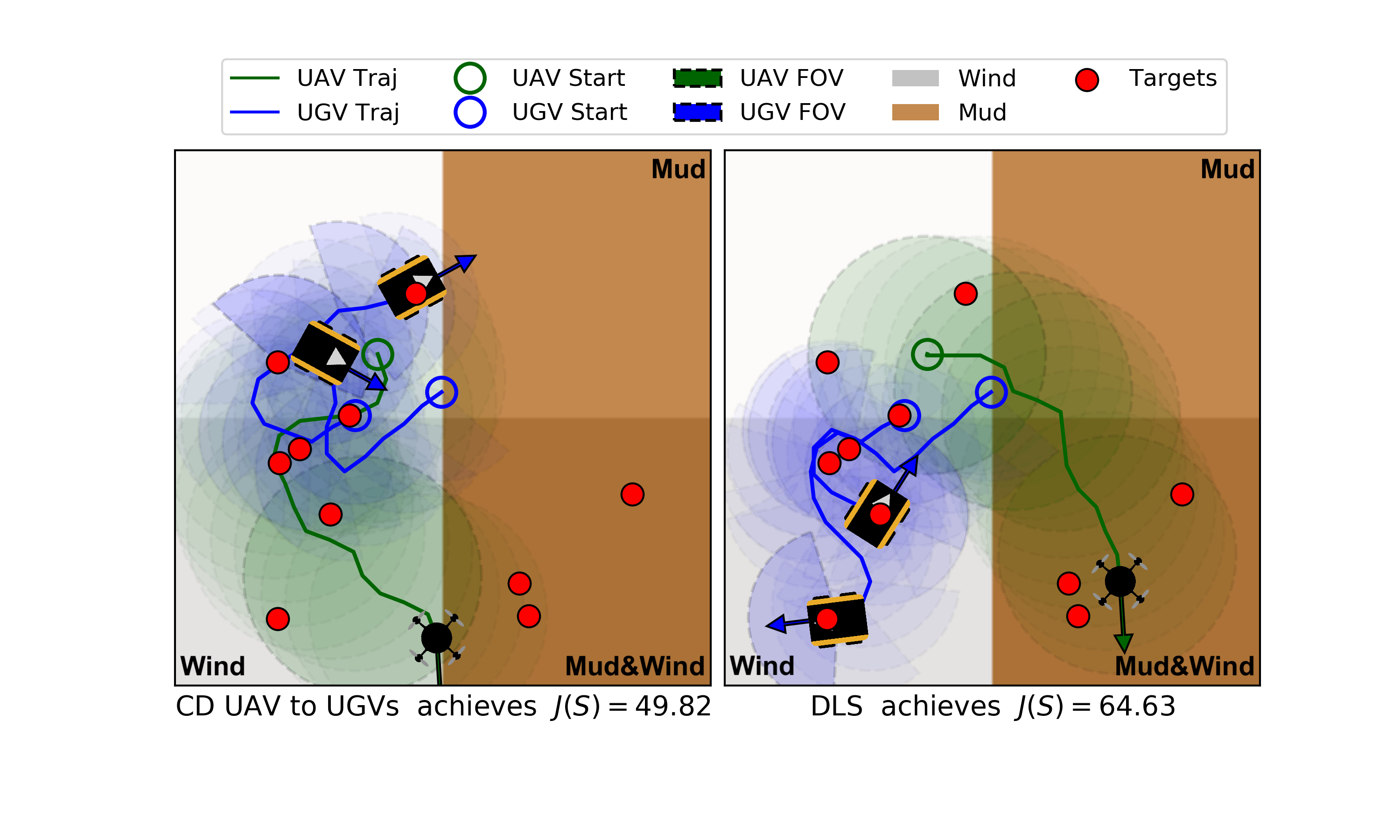}
\caption{Example solutions from \cd{} (left) and \dls{} (right) for 2~UGVs and 1~UAV with $m=0.2$ that penalizes energy cost $C(\mathcal{S})$ in~\eqref{exp:energy_cost}. 
The arena is both windy and muddy, which is costly for the UAV and UGVs, respectively.
(Left) \cd{} performs poorly due to its fixed planning order: the UAV plans first to hover near the targets on the left, rather than venturing over the mud. Thus, the UGVs are under-utilized because they are unwilling to go into the mud to observe the targets on the bottom right. For similar reasons, \cd{} with reversed order under-utilizes the UAV, which is not visualized due to limited space. (Right) In contrast, \dls{} deploys the UAV over the muddy regions, leading to a better value of $J(\mathcal{S})$ in \eqref{eq:sensing_and_energy_objective}.}
\label{fig:hetero_compare}
\end{figure}

Now consider a heterogeneous team with 2 UGVs and 1 UAV with different sensing and control profiles (Table~\ref{tab:robot_setup}) tracking $10$ static targets in a $100\text{ m}\times100\text{ m}$ arena over a longer horizon $T=20$ (see Fig.~\ref{fig:hetero_compare}). The UAV has better sensing range and field of view compared to UGVs, but consumes more energy.
The arena has overlapping muddy and windy regions, so robots must collaboratively decide which should venture into the costly regions.
We set $m_i$ with a common value for every robot $i$ and increase it from $0$ to $0.5$, where each configuration is repeated for $50$ trials. Robots are spawned in the non-muddy, non-windy region, but targets may appear anywhere. We set $\delta=4$ to handle the longer horizon, and evaluate two \cd{} planning orders: from UAV to UGVs, and the reverse.

As shown in Fig.~\ref{fig:hetero_tradeoff}, \dls{} consistently achieves better sensing and energy trade-off than \cd{} on average. To gain intuition, a trial where \cd{} performs poorly is shown in Fig.~\ref{fig:hetero_compare}. Due to the non-monotone objective, the robot who plans first to maximize its own objective can hinder robots who plan later, thus negatively affecting team performance.

\subsection{Hardware-In-The-Loop Benchmark}\label{sec:hil}

\begin{figure} [t!]
\centering
\includegraphics[width=0.9\columnwidth, trim={0.5cm 1.1cm 1cm 1cm},clip]{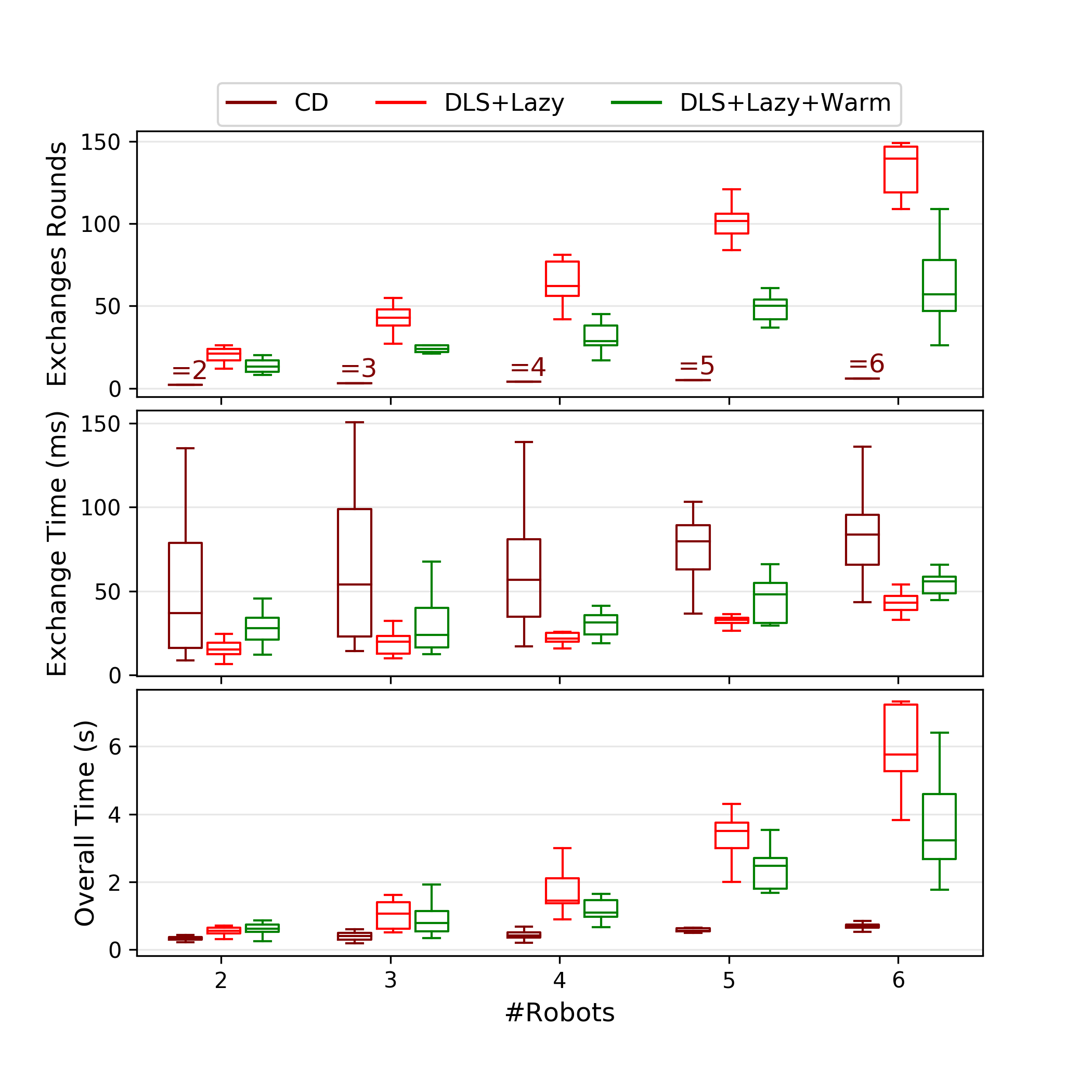}
\caption{Hardware-in-the-loop benchmark on $2$--$6$ synchronized computers (robots), over a ROS communication network with delays. Communication measured in total exchange rounds (top), time per exchange round (middle) and the overall time (bottom). The number of exchange rounds significantly affects the overall planning time in the distributed setting where delays necessitate the use of synchronization with higher overhead. With the proposed greedy warm start strategy, \dls{} planning time can be effectively reduced with fewer exchange rounds. Note that \cd{} is expected to perform the best as it is sequential in nature, but it does not have a worst-case performance guarantee for the non-monotone objective.}
\label{fig:hil_analysis}
\end{figure}

In order to analyze the communication and computation requirements of \dls{} in a real-world distributed setup, we conduct a hardware-in-the-loop benchmark on $2$--$6$ synchronized computers with Intel Core i7 CPUs over a ROS communication network, where every computer acts as a robot and runs a distributed planner. 
This setup is challenging due to the presence of communication delay which increases with the number of robots in the network (about $5~\text{ms}$ when $6$ robots are present), thereby violating the assumption of instantaneous communication required for \dls{}. 
As a result, many robots may broadcast valid proposals to the team, thus requiring a resolution scheme to make sure a common proposal is selected. A simple strategy is to require every robot to receive proposals (\textbf{Delete}, \textbf{Add}, \textbf{Swap} or $\nop$) from the entire team before picking the best proposal, with ties broken in the same fashion (e.g., always favoring the robot with lower index).

As the focus is on communication and computation, we consider a simple scenario of static target tracking with a team of UAVs that are equipped with the same range sensors and require the same energy expenditure. We set $9$ static targets up in a fixed-size arena, and increase the number of robots from $2$ to $6$ with randomized initial conditions.
We focus on the effect of greedy warm start on $\dls{}$ (with lazy method turned on), and also include results for $\cd{}$ for comparison, with $10$ Monte Carlo trials for every method and every number of robots. We use the anytime version of reduced value iteration~\cite{schlotfeldt2018anytime} for ground set generation with better real time performance, which continuously refines its solution by decreasing parameters $\sigma, \epsilon$ within allocated time. Specifically, we use time allocation of $0.1~\text{s}$ (included in the computation time measurement) and initial parameters $\epsilon=8$ with decrements of $0.5$ and $\delta=2.5$ with decrements of $0.05$. The benchmark results are visualized in Fig.~\ref{fig:hil_analysis} which shows the communication requirement (exchange rounds), time spent in each exchange round and the overall planning time. 

The key takeaway is that greedy warm start reduces both the communication and computation when synchronization is required among robots due to communication delays, in contrast to the simulation results in Fig.~\ref{fig:dls_analysis} where greedy warm start does not reduce computation. Therefore, it is necessary to use both lazy and greedy warm start to make $\dls{}$ suitable for real-world multi-robot applications. Note that \cd{} is expected to have the least communication and computation due to its sequential nature, but it does not have a worst-case performance guarantee and achieves worse information and energy trade-offs as shown in Sec.~\ref{sec:exp_multi_dynamic_targets} and Sec.~\ref{sec:exp_hetero}.

\section{Analysis of Control Algorithm}\label{sec:control_results}

The proposed control method in Sec.~\ref{sec:weighted_cbf_qp} is analyzed in simulation where robots with double integrator dynamics have to achieve specified goal states within specified time in 3D. The performances of the centralized and decentralized versions are compared based on final position errors, control efforts (the integral of squared acceleration) and computational time. Note that the baseline of our analysis is the uniform-weight formulation when $\beta=0$ such that $W(\beta)$ is an identity matrix which does not account for timely arrivals. The algorithm is implemented in Python with the CVXOPT~\cite{cvxopt} library on a laptop with an Intel Core i7 CPU.

\begin{figure}[t!]
	\centering
	\begin{subfigure}[b]{0.55\columnwidth}
	    \centering
	   \includegraphics[width=\textwidth, trim={2.7cm 0.4cm 1.9cm 2cm},clip]{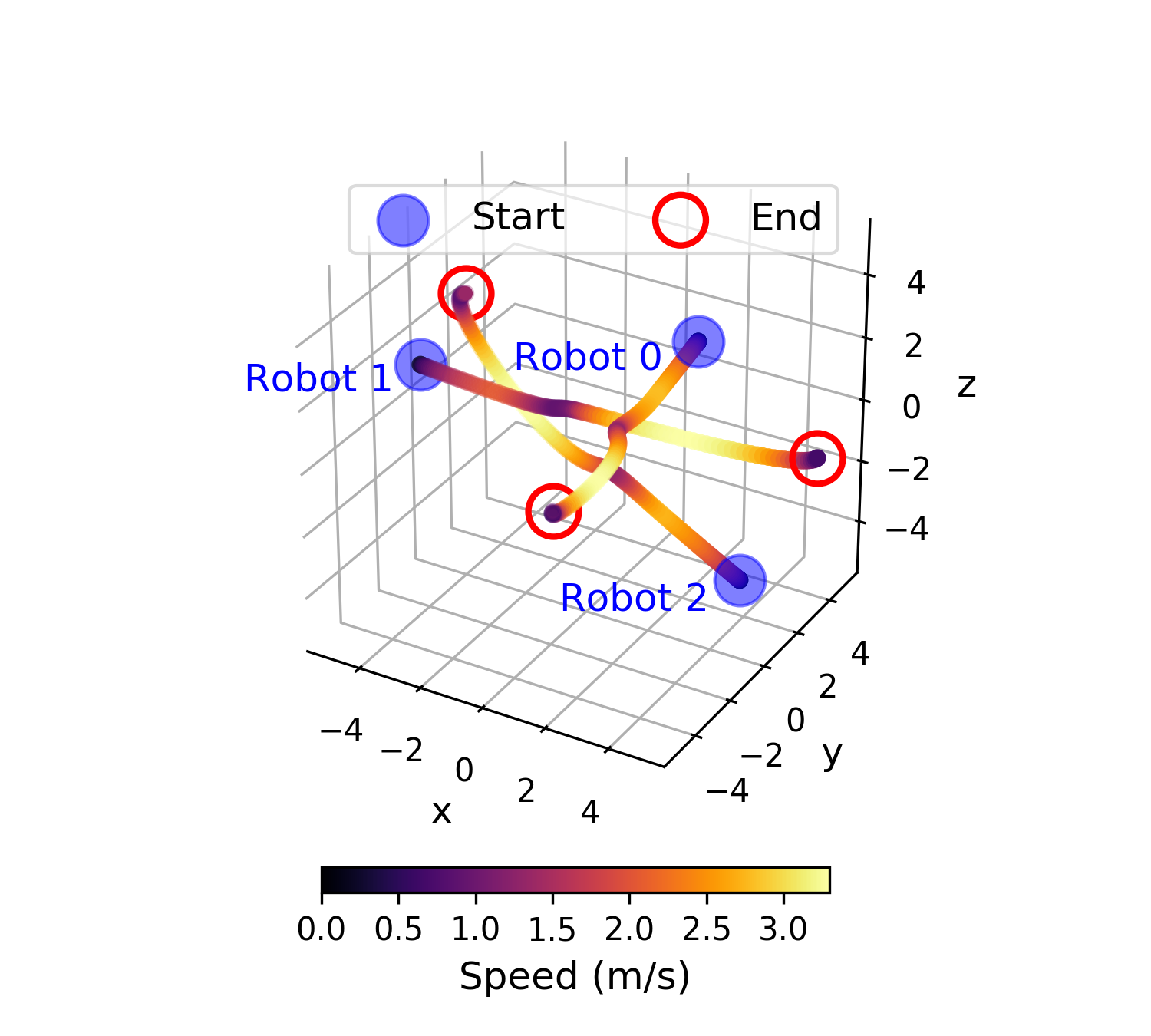}
	    \caption{}\label{fig:cbfqp_benchmark_example_scenario}
	\end{subfigure}
	\begin{subfigure}[b]{0.4\columnwidth}
	    \centering
	    \includegraphics[width=\textwidth, trim={0cm 0.4cm 0.4cm 0.3cm},clip]{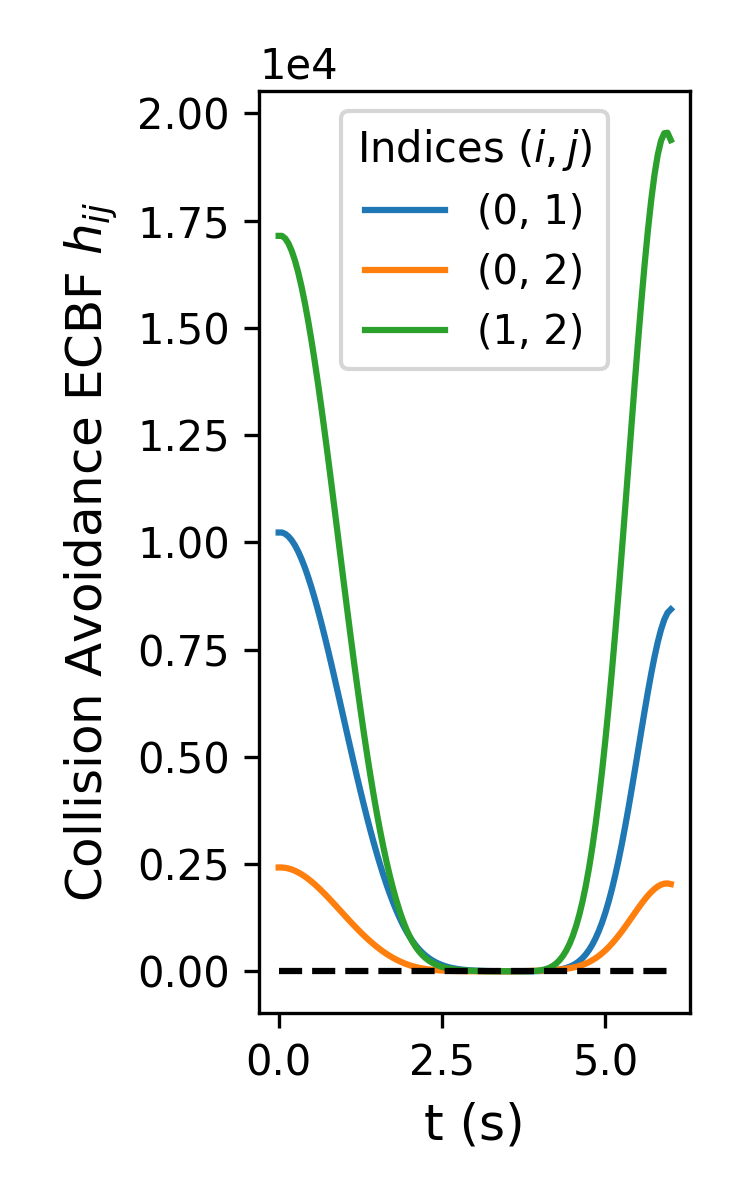}
	    \caption{}\label{fig:cbfqp_benchmark_example_hij}
	\end{subfigure}
	\caption{An example Monte Carlo trial consists of $3$ robots which are spawned on a sphere and try to reach terminal states on the opposite side of the sphere within $6$~s, while satisfying inter-robot collision avoidance constraints encoded by $h_{ij}\geq 0$~\eqref{eq:ecbf_collision_constraint} for robots $i, j\in\{0,1,2\}$ and $i\not= j$. The start and end configurations and the speed profiles are shown in~(\protect\subref{fig:cbfqp_benchmark_example_scenario}). The collision avoidance constraints are satisfied at all time as shown in (\protect\subref{fig:cbfqp_benchmark_example_hij}).}
	\label{fig:cbfqp_benchmark_example}
\end{figure}

\begin{figure}[t!]
	\centering
	\begin{subfigure}[b]{\columnwidth}
	    \centering
	   \includegraphics[width=\textwidth, trim={0.8cm 0cm 0.8cm 0.2cm},clip]{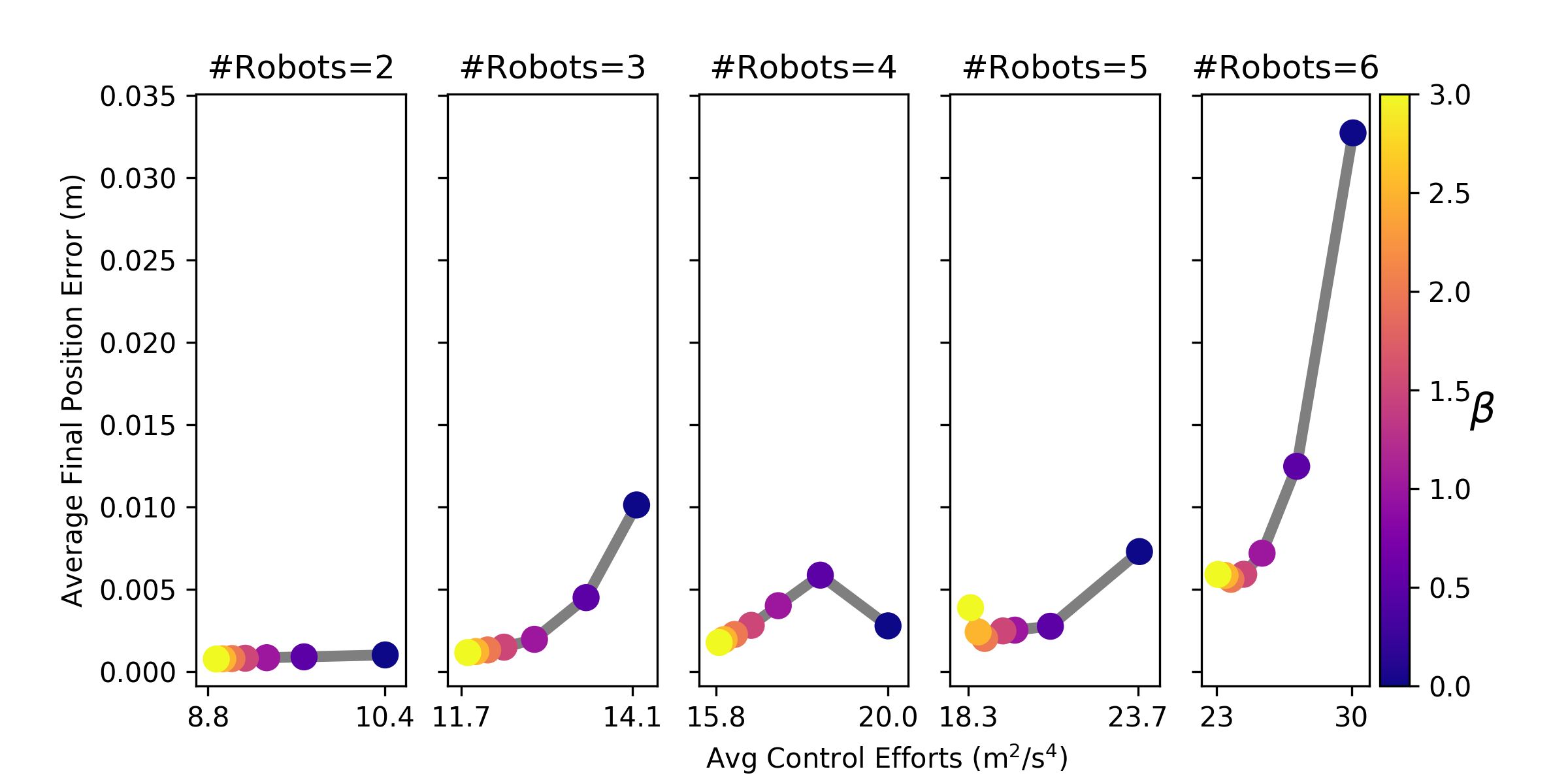}
	    \caption{\scriptsize Centralized CBF-QP with weighted norm.}
	    \label{fig:cent_cbfqp_benchmark}
	\end{subfigure}
	\begin{subfigure}[b]{\columnwidth}
	    \centering
	    \includegraphics[width=\textwidth, trim={0.8cm 0cm 0.8cm 0.2cm},clip]{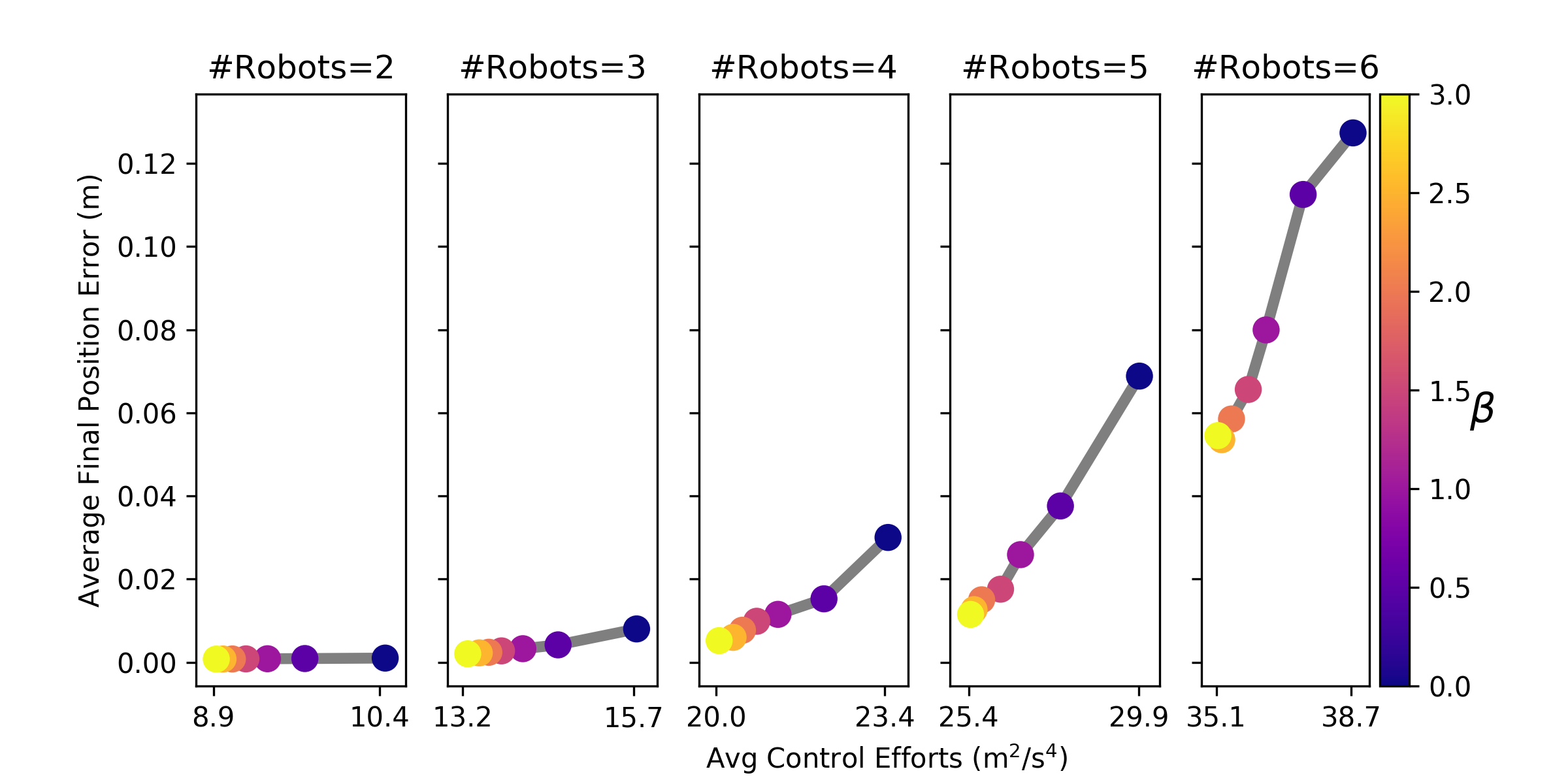}
	    \caption{\scriptsize Decentralized CBF-QP with weighted norm.}
	    \label{fig:dec_cbfqp_benchmark}
	\end{subfigure}
	\caption{Effect of $\beta$ in~\ref{eq:weighted_cbf_qp} that penalizes mission rate deviation in the centralized and decentralized CBF-QP, where each data point is an average over $50$ Monte Carlo trials. The vertical axis shows the average final position error per robot, and the horizontal axis shows the average control effort per robot (integral of squared acceleration). %
    In general, greater $\beta$ decreases both the average control effort and the final position error for both the centralized and decentralized coordination, because the robots spend less time avoiding collisions which allow more time to navigate to the goals with smaller control inputs. 
    Note that this trend does not hold smoothly for some centralized cases, where the terminal errors are already small and sensitive to oscillations due to the collision avoidance constraints. Overall, the decentralized coordination is more conservative and produces higher final position error with higher control effort because each robot has a smaller admissible control space. However, the computational benefit of decentralized coordination (see Fig.~\ref{fig:cbfqp_time}) makes it attractive for real-world multi-robot experiments.}
    \label{fig:cbfqp_benchmark_weighted_norm}
\end{figure}

\begin{figure} [t!]
\centering
\includegraphics[width=0.9\columnwidth, trim={0.4cm 0.5cm 0.4cm 0.5cm}, clip]{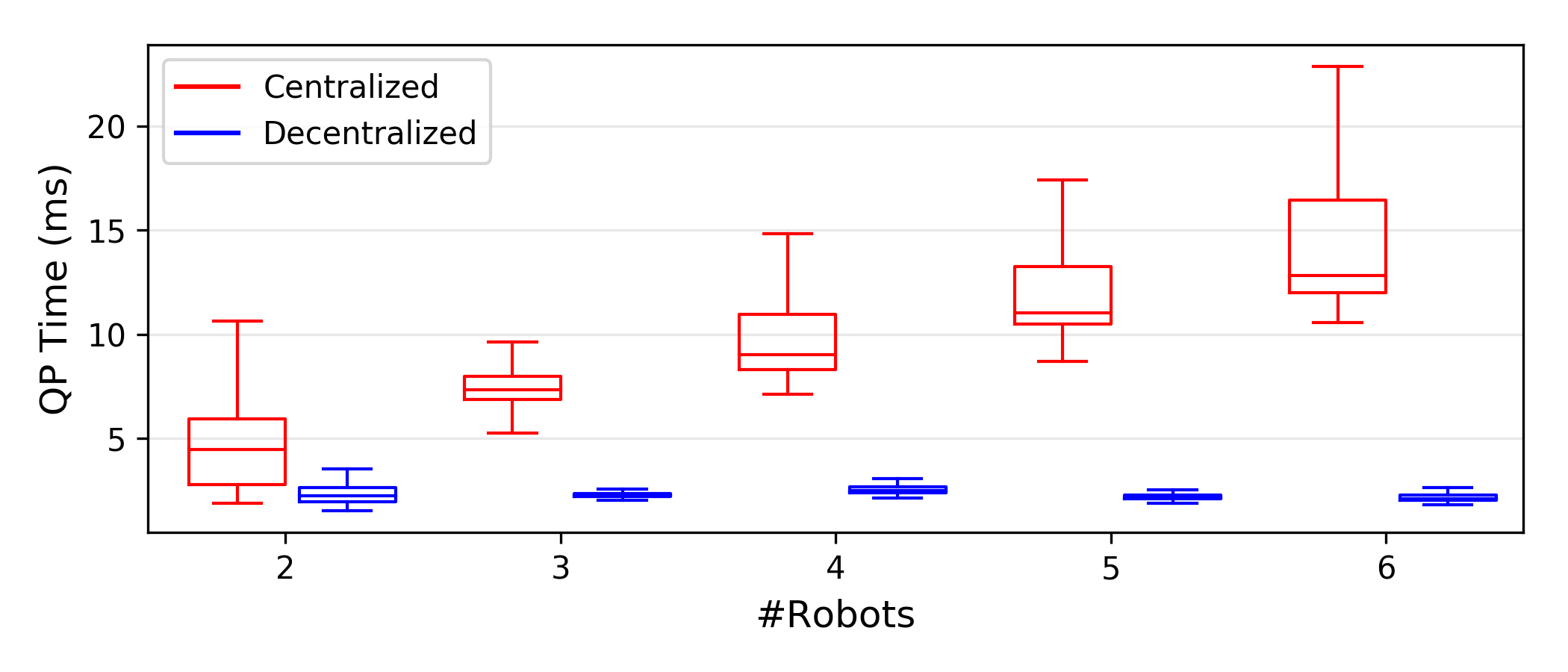}
\caption{Computation time for solving the centralized and the decentralized CBF-QPs.
The number of inter-robot collision avoidance constraints scales linear in the number of robots in the decentralized case (blue) but quadratically in the centralized version (red). }
\label{fig:cbfqp_time}
\end{figure}

\begin{figure*}[t!]
\centering
\includegraphics[width=0.9\textwidth, trim={0cm 0cm 0cm 0cm}, clip]{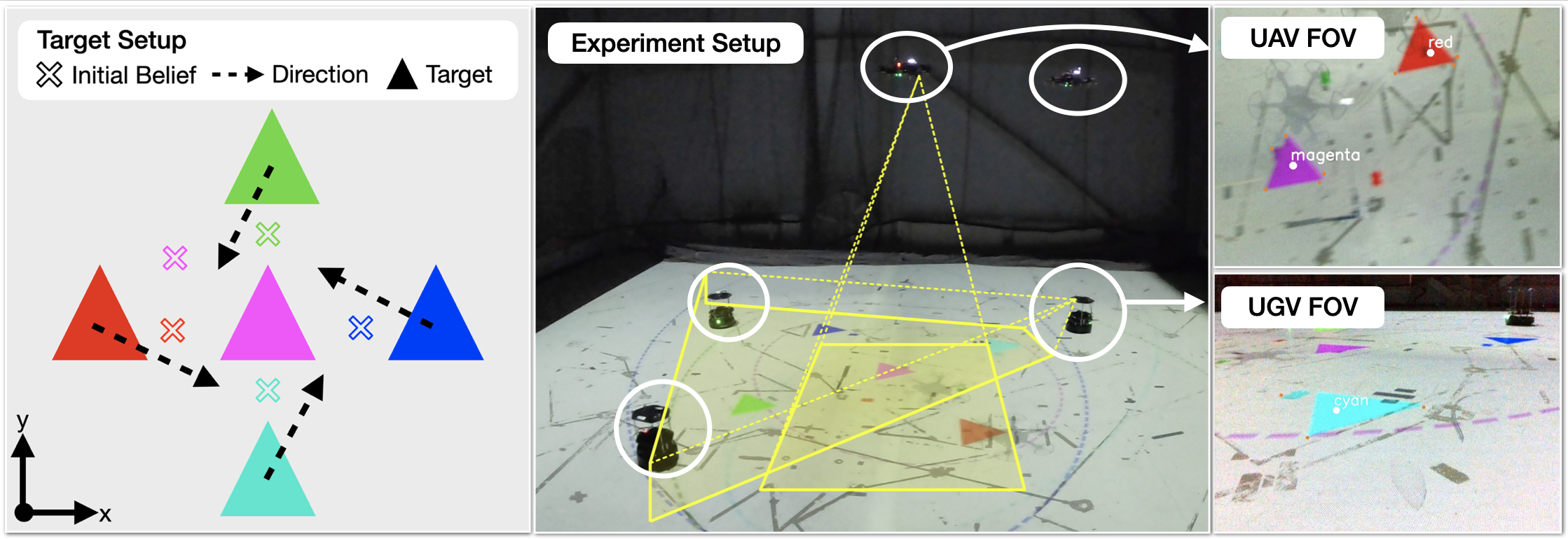}
\caption{The hardware experiments involve projected dynamic targets on the floor (triangles with different colors), which are observed by robots via onboard camera. 
(Left) Targets follow predetermined linear motions except for the magenta target that stays stationary. The ground truth target states are unknown to the robots who only know rough initial locations marked by crosses.
(Middle) The cameras on UAVs and UGVs are downward-facing and forward-facing respectively.
(Right) Targets detections and data associations are based on colors, but the pre-existing markings on the floor and shadows from UAVs make the detection tasks challenging and prone to erroneous detections.}
\label{fig:hw_exp_setup}
\end{figure*}

Robots are randomly spawned on a 3D sphere with radius of $6$~m and are tasked to navigate to the opposite side of the sphere in $6$~s. The safety radius is $D_s=0.5$~m and the z-scale factor is $c=1$. The ECBF parameter $K_{\eta}=[25.5, 10.1]$ is obtained by setting the poles of the closed-loop double integrator system as $[-5, -5.1]$ (see Def.~\ref{def:ECBF}). Furthermore, we impose additional affine constraints in the optimization such that control inputs in acceleration are less than $10~\text{m/s}^2$ in every axis. The specified initial and final conditions of the robots are stationary but we add noise (both in position and velocity) in order to break the symmetry of the problem and reduce the likelihood of deadlocks. Given the number of robots between $2$ to $6$, we also vary the CBF-QP weight $\beta$ between $0$ and $3$ and repeat each configuration for $50$ trials. An example trial with $3$ robots is visualized in Fig.~\ref{fig:cbfqp_benchmark_example} which shows the executed trajectories and safety of the robot via the non-negative barrier function values associated with inter-robot collision avoidance. Note that safety of the robots is satisfied for all trials. However, the decentralized scenario with many more robots and smaller space may lead to infeasibility of the optimization problem, but deadlocks can be ameliorated with higher-level discrete-time planning (energy-aware trajectories are typically spread-out in space) or directly addressed in the QP formulation (e.g., see~\cite{wang2017safetytro}).

The impact of the weight $\beta$ on the norm of the final position error and the control efforts is shown in Fig.~\ref{fig:cbfqp_benchmark_weighted_norm} for both the centralized and the decentralized CBF-QPs. Compared to the commonly used uniform-weight CBF-QP ($\beta=0$), increasing $\beta$ leads to lower average final position error and average control effort, because robots resolve collisions faster and can afford more time to navigate to goal with smaller control inputs. 
Compared to the centralized one, the decentralized controller leads to lower performance due to smaller admissible control space for each robot. However, without the all-to-all communication and potentially higher delays associated with the centralized version, the decentralized controller is attractive for real-world applications by requiring only one-hop communication and lower computational overhead (see Fig.~\ref{fig:cbfqp_time}).

\section{Hardware Experiments---Dynamic Target Tracking}\label{sec:hardware_experiments}

\begin{figure*}[t!]
\centering
\includegraphics[width=0.9\textwidth, trim={0cm 0cm 0cm 0cm}, clip]{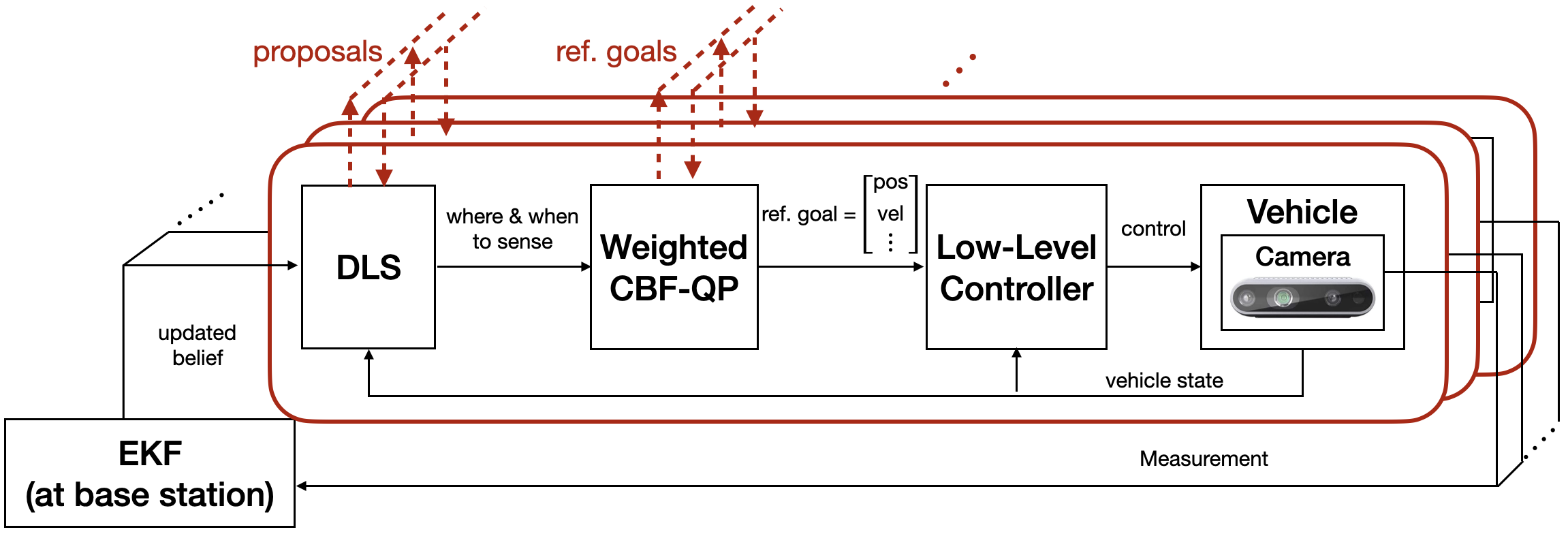}
\caption{Overall information gathering architecture that includes onboard planning, control and sensing (shown in red blocks). For convenience, measurements are fused at the base station for visualization, but decentralized estimators can be readily substituted and are not the focus of this work. The robots exchange proposals for planning high-level waypoints using \dls{} or \cd{}, which are tracked by the downstream safety controller based on~\ref{eq:weighted_cbf_qp}. The safety controller requires reference goals (e.g., position and velocity for double integrators) from other robots to achieve collision avoidance. Lastly, low-level controllers stabilize the robots around the reference goals.}
\label{fig:hw_architecture}
\end{figure*}

The previous sections provide analysis on the planning and control strategies individually in simulation or hardware-in-the-loop tests. It is also crucial to ensure the overall approach is feasible and produces good performance under real-world factors such as occlusion, collision avoidance, limited onboard computation, and extended mission horizon that necessitates re-planning. To this end, the proposed approach is tested on hardware for a dynamic target tracking task, where a group of ground and aerial robots track the states of moving targets with onboard sensors while ensuring safety. 
Next, details about the target simulation, robot characteristics and overall approach are discussed in Sec.~\ref{sec:hw_exp_setup}. Configurations for the planner and controller are discussed in Sec.~\ref{sec:hw_planning_and_control_setup}. Lastly, the performance and feasibility of the proposed approach are analyzed in Sec.~\ref{sec:hw_performance_analysis}.

\subsection{Target Model, Observation Model, and Hardware Setup}\label{sec:hw_exp_setup}
An overview of the target models, experiment setup, and detection techniques are summarized in Fig.~\ref{fig:hw_exp_setup}. To simulate dynamic targets, an overhead projection system visualizes moving shapes with different colors with known fixed sizes. For simplicity, four targets follow pre-determined linear trajectories with constant velocity of $0.15$~{m/s} across the floor, and one target remains stationary in the center over the span of 70~s. Robots are only given a rough position estimations at the beginning of the mission with large uncertainty. The targets are modeled as double integrator and the robots need to reduce uncertainty about the target position and velocities. 

A fleet of three UGVs and two UAVs are equipped with computers with Intel Core i7 processor and RGB cameras to detect targets based on colors. The ground truth poses of the robots are provided by a Vicon tracking system. Given a camera pose and a pixel coordinate that represents target position in the image space, a unique position in the world frame can be deduced since the targets are restricted to the floor. To make the problem more difficult and encourage more movements, we only use \textit{range measurements} instead of position measurements directly. Note that the detection task is challenging due to existing markings on the floor and shadows from the UAVs that can fragment a color patch. To resolve this issue, the convex hull of all pixels of a given target color is computed and prior knowledge about target sizes and room bounds is used to reduce erroneous detections. The robots process the images asynchronously and measurements are sent to base station where the centralized filter runs at 3~Hz.

\subsection{Planner and Controller Configurations}\label{sec:hw_planning_and_control_setup}

A schematic for the overall planning and control architecture is shown in Fig.~\ref{fig:hw_architecture}, where the robots have to exchange proposals for planning informative trajectories in a distributed fashion and exchange reference goals (e.g., position, velocities for double integrators) in order to avoid collisions in a decentralized fashion. 
For simplicity, target estimation is achieved via a centralized filter at the base station to ensure that robots share the same target belief, but distributed estimation methods such as the distributed Kalman filter can be used instead (e.g., see~\cite{schlotfeldt2018anytime} for a concrete example and a discussion on communication requirements and convergence guarantees).

During planning, UAVs are restricted to $3$~m height and UGVs are on the floor, and we use the same differential-drive kinematic model
for UAVs and UGVs to generate high-level waypoints for simplicity, with sampling period $\tau=3$~s. The motion primitives for UGVs are $\{u=(\nu, \omega) \mid \nu \in \{0.3,0.6\}~\text{m/s},\ \omega\in\{0,\pm0.2, \pm 0.5\}\text{ rad/s}\}$, and the ones for UAVs are more aggressive: $\{u=(\nu, \omega) \mid \nu \in \{0.3,0.5, 0.8\}~\text{m/s},\ \omega\in\{0,\pm0.35, \pm 0.5, \pm0.75\}\text{ rad/s}\}$. The range measurement noise standard deviations are set to $0.4\text{ m}$ for UAVs and $1.0\text{ m}$ for UGVs based on empirical evaluations.
For ground set generation, anytime reduced value iteration~\cite{schlotfeldt2018anytime} is used to plan trajectories with horizon $T=8$ with time allocation of $0.3~\text{s}$, initial parameters $\epsilon=5$ with decrements of $0.5$ and $\delta=1.5$ with decrements of $0.05$. To limit the overall planning time, we limit the ground set of every robot to be fewer than 800 trajectories.
When used in the planning objective, the energy cost for each generated trajectory is computed in closed form according to~\eqref{eq:optimal_lqr_objective} and weighted by $0.1$ to make the energy cost term comparable to the information gain.
To account for updated target beliefs and long mission time span, replanning is scheduled every two time steps (6~s) and planning occurs simultaneously while the robots execute plans from previous rounds.
Lastly, both lazy evaluation and greedy warm start are used during \dls{}.

\begin{figure*} [t!]
\centering
\includegraphics[width=\textwidth, trim={0cm 0cm 0cm 0cm}, clip]{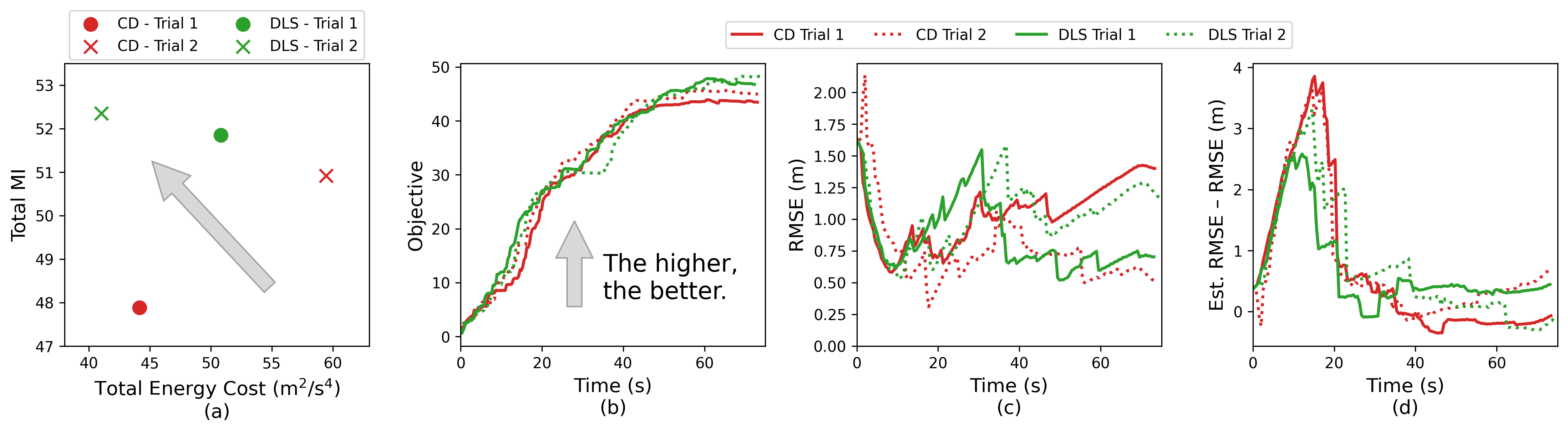}
\caption{Performance analysis of the hardware experiments. (a) Trade-off between sensing performance (measured by mutual information (MI)) and energy cost using \dls{} and \cd{}, where it is better to be in the upper left indicated by the gray arrow. 
(b) The trade-off objective being maximized. \dls{} outperforms \cd{} by achieving better objective values.
(c) The root mean square errors (RMSEs) of target positions over time. The initial decrease and the subsequent increase in tracking performance are expected due to the change in mission difficulty.
(d) The differences between the filter-estimated RMSEs and the ground truth RMSEs. The values stay bounded near $0$ towards the end of the mission, indicating that the filter correctly reports high uncertainty as tracking performance degrades.
}
\label{fig:hw_tradeoff_obj_rmse}
\end{figure*}

For the controller, we use double integrator models for both the UAVs and UGVs, because multi-rotors and differential drive robots are both differentially flat
and move at low speed. The integrator states (position and velocities) will be tracked by lower-level controllers (position controller for UAVs~\cite{lee2010geometric} and Lyapunov-based pose controller for UGVs). The inter-UGV safety distance is $1.0$~m and the one for UAVs is $1.5$~m. We further restrict robots within the arena, limit their velocities and accelerations (control input) via ECBF. For the decentralized ECBF constraint~\eqref{eq:ecbf_bij}, poles of $[-3.0, -3.1]$ are used to generate the $K_{\eta}=[9.3, 6.1]$ vector, and any pair of robots $i$ and $j$ share equal responsibility for avoiding collisions (i.e., $\alpha_i=\alpha_j>0$). The weight matrix $R$ in the LQR energy objective~\eqref{eq:lqr} is set to identity. To ensure timely arrival at desired sensing configurations without overly aggressive maneuvers, we set $\beta=0.5$ for~\ref{eq:weighted_cbf_qp} based on empirical evaluations.

\subsection{Performance Analysis}\label{sec:hw_performance_analysis}

\begin{figure} [t!]
\centering
\includegraphics[width=\columnwidth, trim={0cm 0cm 0cm 0cm}, clip]{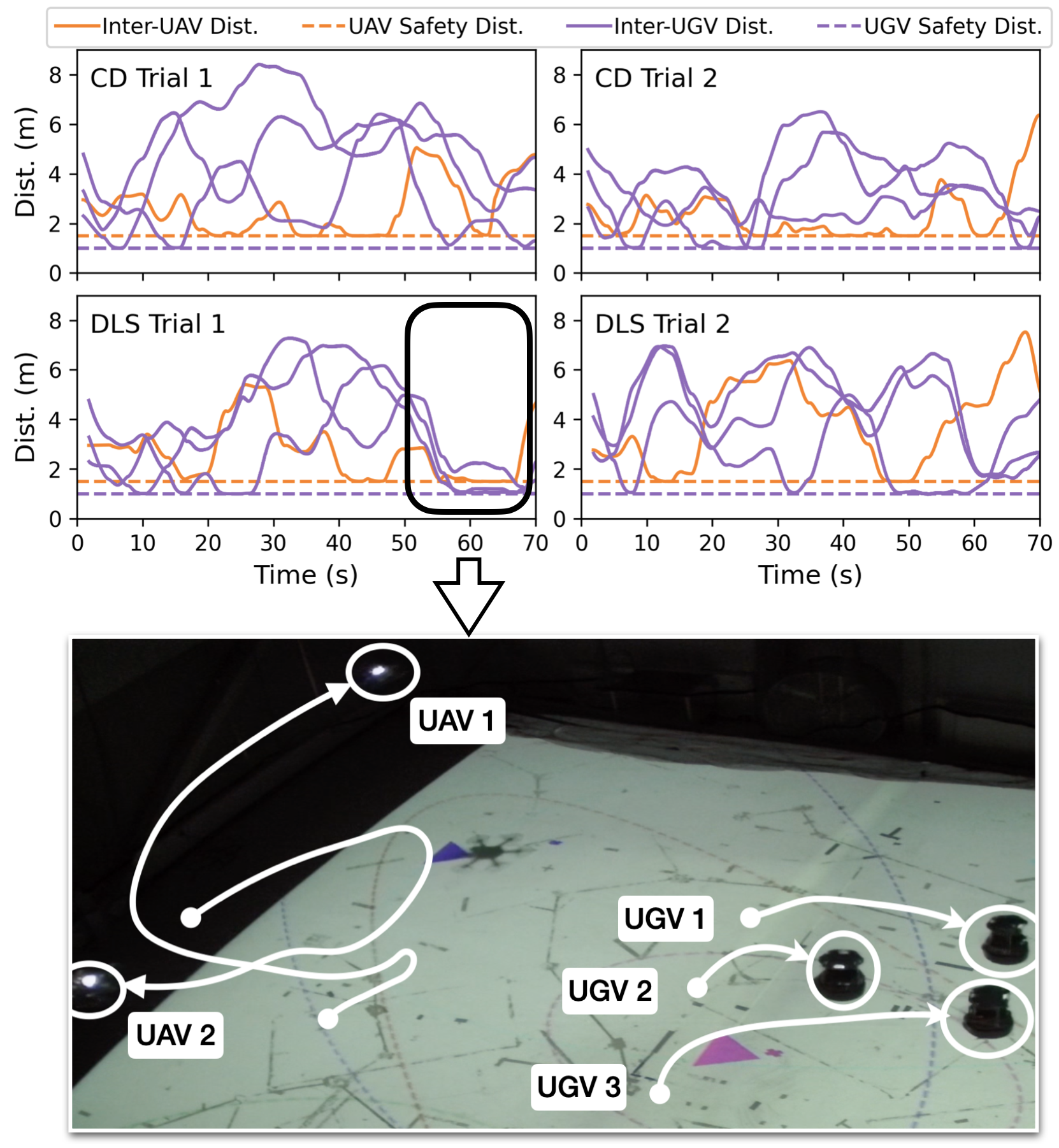}
\caption{Collision avoidance constraints are satisfied at all time in all hardware trials, where the inter-UAV and inter-UGV distances never fall below the safety distances (top 4 figures). A particular moment in Trial 1 with \dls{} is shown at the bottom with robot trajectories overlaid on top, demonstrating safe decentralized navigation.}
\label{fig:hw_collision_avoidance}
\end{figure}

\begin{table}[t]
\caption{Computation and communication of proposed approach.}\label{tab:hw_comp_comm_table}
\begin{tabular}{lllllll}
\hline
\multirow{2}{*}{} & \multicolumn{2}{l}{Planning Time (s)} & \multicolumn{2}{l}{\#Exchanges} & \multirow{2}{*}{\begin{tabular}[c]{@{}l@{}}Exchange\\ Delays (ms)\end{tabular}} & \multirow{2}{*}{\begin{tabular}[c]{@{}l@{}}Control\\ Time (ms)\end{tabular}} \\ \cline{2-5}
                  & \cd{}                & \dls{}               & \cd{}           & \dls{}              &                                                                                 &                                                                              \\ \hline
mean              & 0.995             & 1.656             & 5            & 17.296           & 6.696                                                                           & 2.650                                                                        \\ \hline
std               & 0.092             & 0.226             & 0            & 3.400            & 2.145                                                                           & 1.082   \\ \hline
\end{tabular}
\end{table}

The performances of \dls{} and \cd{} are measured in two representative trials with different initial robot positions, and the computation and communication requirements of the overall approach are summarized in Table~\ref{tab:hw_comp_comm_table}. Even with an average communication delay about 7~ms, the robots are able to reliably plan and remain safe. 
Note that the average time to solve for the decentralized weighted CBF-QP still remains tractable and well under 5~ms, and the inter-robot distances always remain above safety threshold for all hardware trials, as visualized in Fig.~\ref{fig:hw_collision_avoidance}.

Although \dls{} requires more computation and communication than \cd{}, it provides a worst-case performance guarantee while optimizing the trade-off between sensing and energy cost. Consistent with the simulation results in Fig.~\ref{fig:hetero_tradeoff}, \dls{} achieves better trade-off between information gain and energy cost than the trade-off achieved by \cd{} after averaging the results from 2 trials, as shown in Fig.~\ref{fig:hw_tradeoff_obj_rmse}a. In addition, Fig.~\ref{fig:hw_tradeoff_obj_rmse}b shows that \dls{} achieves a better objective value (weighted sum of information gain and energy cost) than \cd{}, which is consistent with the simulation results in Fig.~\ref{fig:dls_obj_control_sensing_time}.

To gain insight into the actual tracking performance achieved by the robots, we plot the root mean square errors (RMSEs) of the estimated target positions in Fig.~\ref{fig:hw_tradeoff_obj_rmse}c. Because the mission starts easy as targets gather in the center initially and becomes more difficult as they move towards the boundary, the tracking performance first improves and subsequently degrades as expected. 
Other than the accuracy of the estimated target positions, we also analyze whether the filter predicts higher uncertainty when the tracking performance degrades. To this end, Fig.~\ref{fig:hw_tradeoff_obj_rmse}d plots the difference between filter-estimated RMSEs (computed as the root mean of the trace of the covariance matrices for the target positions) and the true RMSEs. The fact that the differences eventually stay bounded near $0$ indicates that the filter correctly reports high uncertainty when the tracking task becomes more challenging.

Overall, the proposed hierarchical approach that separates planning and control allows both \dls{} and \cd{} to achieve reasonable target tracking performance in hardware with noisy perception systems while satisfying collision-avoidance constraints. Moreover, the hardware trials have demonstrated that the proposed planner \dls{} is a feasible strategy for real-world applications that can achieve better trade-off between information gain and energy cost compared to \cd{}, which does not have any worst-case performance guarantee when optimizing the trade-off objective.

\section{Conclusion}\label{sec:conclusion}
This work proposed a non-monotone information gathering method which consists of a distributed planner that optimizes the trade-off between sensing performance and energy expenditure and a decentralized controller that ensures safety and encourages timely arrival at designated sensing configurations. The proposed approach was analyzed in parts and as a whole via simulations, hardware-in-the-loop benchmarks, and a real-world dynamic target tracking mission, while outperforming the state-of-the-art coordinate descent method. Although feasible in practice, \dls{} still requires all-to-all communication and its computation scales poorly with large number of robots. An interesting future direction is to reduce the complexity of \dls{} by exploiting spatial separation and analyze the impact on the performance guarantee.
Another promising direction is to design approximation algorithms that ensure collision avoidance during planning while still leveraging the submodular property of the objective, which may further improve the overall performance compared to the proposed method that avoids collisions exclusively via the low-level controller.
Lastly, this work assumes that the robots know each other's states perfectly but this strong assumption may not hold in practice. Therefore, it is important to handle imperfect state estimation and analyze its impact on the performance and safety guarantees of this work.

\appendices
\section{Proof of Prop.~\ref{prop:ecbf_affine_term}}\label{appendix:ecbf_affine_term_proof}
Recall that $u_i=p_i^{(r)}$ and $p_i=[ \ctnsax{p}{i}{x}, \ctnsax{p}{i}{y}, \ctnsax{p}{i}{z}]\tr$, and we denote $\Delta_u=u_i-u_j=[\Delta_x^{(r)}, \Delta_y^{(r)}, c\Delta_z^{(r)}]\tr$.
For the base case $r=1$, it can be shown that 
\begin{align}
\dot h_{ij} &= 4(\Delta_x^2+\Delta_y^2)(\Delta_x\dot\Delta_x+\Delta_y\dot\Delta_y) + 4\Delta_z^3\dot\Delta_z\\
&= 4\left[(\Delta_x^2+\Delta_y^2)\Delta_x, (\Delta_x^2+\Delta_y^2)\Delta_y, \Delta_z^3/c\right] \cdot \Delta_u,
\end{align}
where $L_fh_{ij}(x_i,x_j)=0$. 

Next, assume that Prop.~\ref{prop:ecbf_affine_term} holds for $r=\tilde r > 1$, \ie,
\begin{align}
h_{ij}^{(\tilde r)} &=A_{ij} \Delta_u + L_f^{\tilde r}h_{ij} \\
&=4(\Delta_x^2+\Delta_y^2)(\Delta_x\Delta_x^{(\tilde r)}+\Delta_y\Delta_y^{(\tilde r)}) + 4\Delta_z^3\Delta_z^{(\tilde r)} + L_f^{\tilde r}h_{ij}.
\end{align}
Taking the $(\tilde r+1)$-th derivative of $h_{ij}$, we get
\begin{align}
& h_{ij}^{(\tilde r + 1)} = 4(2\Delta_x\dot\Delta_x+2\Delta_y\dot\Delta_y)(\Delta_x\Delta_x^{(\tilde r)}+\Delta_y\Delta_y^{(\tilde r)}) \nonumber\\
& +4(\Delta_x^2+\Delta_y^2)(\dot\Delta_x\Delta_x^{(\tilde r)}+\Delta_x\Delta_x^{(\tilde r+1)} +\dot\Delta_y\Delta_y^{(\tilde r)}+\Delta_y\Delta_y^{(\tilde r+1)}) \nonumber\\
& + 12\Delta_z^2\dot\Delta_z\Delta_z^{(\tilde r)} + 4 \Delta_z^{3}\Delta_z^{(\tilde r+1)} + \frac{d}{dt}\left(L_f^{\tilde r}h_{ij}\right)\\
& =  4(\Delta_x^2+\Delta_y^2)(\Delta_x\Delta_x^{(\tilde r+1)}+\Delta_y\Delta_y^{(\tilde r+1)})+ 4\Delta_z^3\Delta_z^{(\tilde r+1)} \nonumber\\
&+ L_f^{\tilde r + 1}h_{ij} \nonumber \\
& = \underbrace{A_{ij}\left(p_i^{(\tilde r +1)}-p_j^{(\tilde r +1)}\right)}_{=A_{ij}(u_i-u_j)} + L_f^{\tilde r + 1}h_{ij},
\end{align}
where 
\begin{align}
L_f^{\tilde r + 1}h_{ij} = & 4(2\Delta_x\dot\Delta_x+2\Delta_y\dot\Delta_y)(\Delta_x\Delta_x^{(\tilde r)}+\Delta_y\Delta_y^{(\tilde r)}) \nonumber\\
& +4(\Delta_x^2+\Delta_y^2)(\dot\Delta_x\Delta_x^{(\tilde r)}+\dot\Delta_y\Delta_y^{(\tilde r)}) \nonumber\\
& + 12\Delta_z^2\dot\Delta_z\Delta_z^{(\tilde r)} + \frac{d}{dt}\left(L_f^{\tilde r}h_{ij}\right).
\end{align}

\bibliographystyle{IEEEtran}
\bibliography{bibs}

\begin{IEEEbiography}[{\includegraphics[width=1in, height=1.25in, clip,keepaspectratio, trim={3cm 2cm 3cm 2cm}]{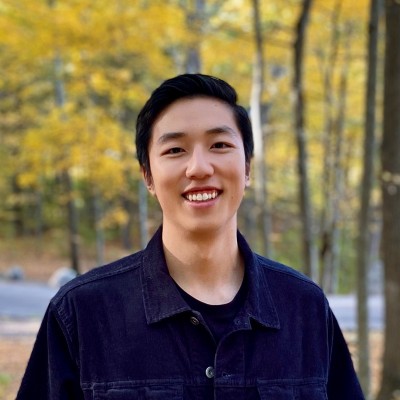}}]{Xiaoyi Cai}
received the B.S. and M.S. degrees in
electrical and computer engineering from the Georgia Institute of Technology, Atlanta, GA, USA, in
2017 and 2019, respectively. He is currently working towards his Ph.D. at the Department of Aeronautics and Astronautics, Massachusetts Institute of Technology, Cambridge, MA, USA.
His research interests include cooperative control
and planning for autonomous multi-robot systems, risk-aware motion planning, and traversability analysis for off-road navigation.
\end{IEEEbiography}

\begin{IEEEbiography}[{\includegraphics[width=1in, height=1.25in, clip,keepaspectratio, trim={0cm 1cm 0.5cm 1.0cm}]{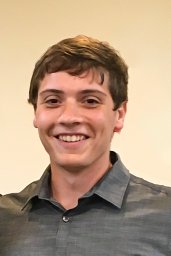}}]{Brent Schlotfeldt}
received the Ph.D. degree in Electrical and Systems Engineering and the Master of Science in Robotics from the University of Pennsylvania, Philadelphia, PA, in 2021 and 2017, respectively, and the Bachelor of Science in Electrical Engineering and Computer Science from the University of Maryland, College Park, MD, USA, in 2016. 
His research interests include planning, control, and state estimation for robotic systems, with applications to autonomous driving and active sensing.
\end{IEEEbiography}

\begin{IEEEbiography}[{\includegraphics[width=1in, height=1.25in, clip,keepaspectratio, trim={0.7cm 0cm 0.8cm 0.3cm}]{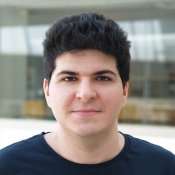}}]{Kasra Khosoussi}
received the B.Sc. degree in computer engineering from the Department of Electrical and Computer Engineering at the K. N. Toosi University of Technology, Tehran, Iran, in 2011, and the Ph.D. degree in robotics from the University of Technology Sydney, Australia, in 2017. He is currently a Senior Research Scientist at CSIRO and a Visiting Fellow at the Australian National University. Previously he was a Research Scientist (2019-2021) and a Postdoctoral Associate (2017–2018) at MIT’s Department of Aeronautics and Astronautics and Laboratory for Information and Decision Systems. His research is primarily focused on developing robust and reliable algorithms with provable performance guarantees for single- and multi-robot perception and autonomy. He is a finalist for the ICRA 2018 Best Paper Award for Multi-Robot Systems, an honorable mention for the 2021 IEEE Transactions on Robotics King-Sun Fu Memorial Best Paper Award, and the winner of Hilti SLAM Challenge at ICRA 2022.
\end{IEEEbiography}

\begin{IEEEbiography}[{\includegraphics[width=1in, height=1.25in, clip,keepaspectratio, trim={10cm 40cm 14.5cm 0cm}]{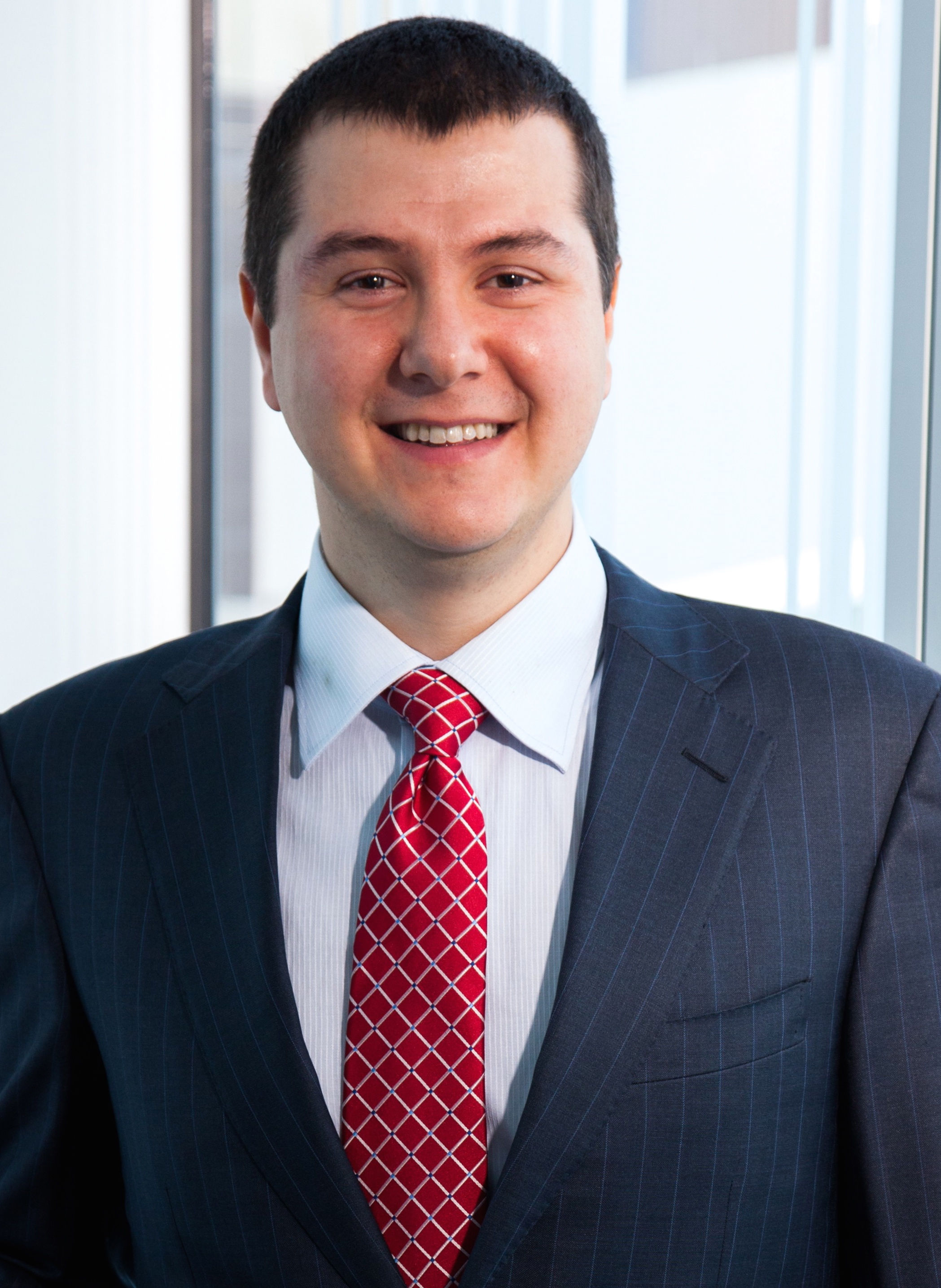}}]{Nikolay Atanasov}
(S’07-M’16) is an Assistant Professor of Electrical and Computer Engineering at the University of California San Diego, La Jolla, CA, USA. He obtained a B.S. degree in Electrical Engineering from Trinity College, Hartford, CT, USA in 2008, and M.S. and Ph.D. degrees in Electrical and Systems Engineering from University of Pennsylvania, Philadelphia, PA, USA in 2012 and 2015, respectively. His research focuses on robotics, control theory, and machine learning with applications to active perception problems for autonomous mobile robots. He works on probabilistic models that unify geometric and semantic information in simultaneous localization and mapping (SLAM) and on optimal control and reinforcement learning algorithms for minimizing uncertainty in probabilistic models. Dr. Atanasov's work has been recognized by the Joseph and Rosaline Wolf award for the best Ph.D. dissertation in Electrical and Systems Engineering at the University of Pennsylvania in 2015, the best conference paper award at the IEEE International Conference on Robotics and Automation (ICRA) in 2017, and the NSF CAREER award in 2021.
\end{IEEEbiography}

\begin{IEEEbiography}[{\includegraphics[width=1in, height=1.25in, clip,keepaspectratio]{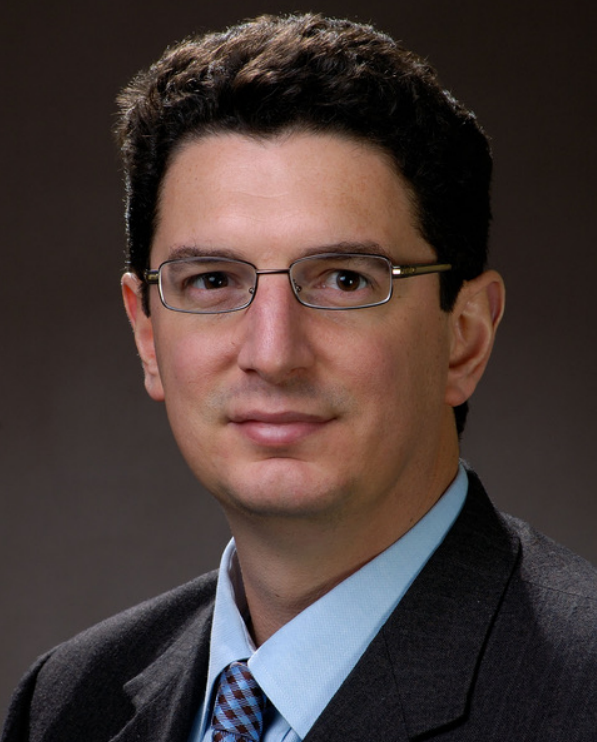}}]{George J. Pappas}
(S’90-M’91-SM’04-F’09) received the Ph.D. degree in electrical engineering and computer sciences from the University of California, Berkeley, CA, USA, in 1998. He is currently the UPS Foundation Professor and Chair of the Department of Electrical and Systems Engineering, University of Pennsylvania, Philadelphia, PA, USA. He also holds a secondary appointment with the Department of Computer and Information Sciences and the Department of Mechanical Engineering and Applied Mechanics. He is a Member of the GRASP Lab and the PRECISE Center. He had previously served as the Deputy Dean for Research with the School of Engineering and Applied Science. His research interests include control systems, robotics, and machine learning. Dr. Pappas has received various awards, such as the Antonio Ruberti Young Researcher Prize, the George S. Axelby Award, the Hugo Schuck Best Paper Award, the George H. Heilmeier Award, the National Science Foundation PECASE award and numerous best student papers awards.
\end{IEEEbiography}

\begin{IEEEbiography}[{\includegraphics[width=1in, height=1.25in, clip,keepaspectratio, trim={2.5cm 2cm 3cm 0cm}]{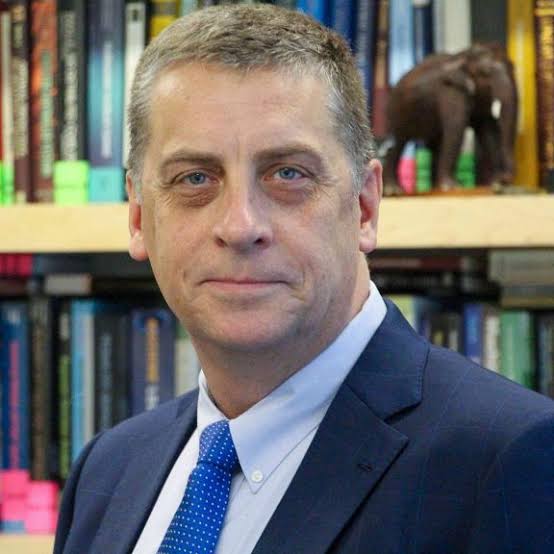}}]{Jonathan P. How}
received the B.A.Sc. degree in aerospace option from the University of Toronto, Toronto, ON, Canada, in 1987, and the S.M. and Ph.D. degrees in aeronautics and astronautics from Massachusetts Institute of Technology, Cambridge, MA, USA, in 1990 and 1993, respectively.
He is the Richard C. Maclaurin Professor of aeronautics and astronautics with the Massachusetts Institute of Technology. Prior to joining MIT in 2000, he was an Assistant Professor with the Department of Aeronautics and Astronautics, Stanford University. His research focuses on robust planning and learning under uncertainty with an emphasis on multiagent systems.
Dr. How was the Editor-in-Chief of the IEEE Control Systems Magazine (2015–2019) and was elected to the Board of Governors of the IEEE Control System Society (CSS) in 2019. His work has been recognized with multiple awards, including the 2020 IEEE CSS Distinguished Member Award, the 2020 AIAA Intelligent Systems Award, the 2015 AeroLion Technologies Outstanding Paper Award for Unmanned Systems, the 2015 IEEE CSS Video Clip Contest, the 2011 IFAC Automatica award for best applications paper, and the 2002 Institute of Navigation Burka Award. He also received the Air Force Commander’s Public Service Award in 2017. He is a Fellow of AIAA and was elected to the National Academy of Engineering in 2021.
\end{IEEEbiography}

\vfill

\end{document}